\documentclass{article} 
\usepackage{fullpage,times}
\usepackage{parskip}

\usepackage{user}

\title{\resizebox{1.0\textwidth}{!}{\textbf{A Stability Analysis of Fine-Tuning a Pre-Trained Model}}}

\author{
    \textbf{Zihao Fu,\textsuperscript{\rm 1} 
    Anthony Man-Cho So,\textsuperscript{\rm 2}
    Nigel Collier\textsuperscript{\rm 1}}
}

\date{
    \textsuperscript{\rm 1}Language Technology Lab, University of Cambridge,\\
    \textsuperscript{\rm 2}The Chinese University of Hong Kong\\
    \{zf268,nhc30\}@cam.ac.uk,
    manchoso@se.cuhk.edu.hk
}

\begin{document}

\maketitle

\begin{abstract}
    Fine-tuning a pre-trained model (such as BERT, ALBERT, RoBERTa, T5, GPT, etc.) has proven to be one of the most promising paradigms in recent NLP research. However, numerous recent works indicate that fine-tuning suffers from the instability problem, i.e., tuning the same model under the same setting results in significantly different performance. Many recent works have proposed different methods to solve this problem, but there is no theoretical understanding of why and how these methods work. In this paper, we propose a novel theoretical stability analysis of fine-tuning that focuses on two commonly used settings, namely, full fine-tuning and head tuning. We define the stability under each setting and prove the corresponding stability bounds. The theoretical bounds explain why and how several existing methods can stabilize the fine-tuning procedure.
    In addition to being able to explain most of the observed empirical discoveries, our proposed theoretical analysis framework can also help in the design of effective and provable methods. Based on our theory, we propose three novel strategies to stabilize the fine-tuning procedure, namely, Maximal Margin Regularizer (MMR), Multi-Head Loss (MHLoss), and Self Unsupervised Re-Training (SURT). We extensively evaluate our proposed approaches on 11 widely used real-world benchmark datasets, as well as hundreds of synthetic classification datasets. The experiment results show that our proposed methods significantly stabilize the fine-tuning procedure and also corroborate our theoretical analysis.    
\end{abstract}
  
\section{Introduction}

Fine-tuning a pre-trained model (such as BERT \cite{devlin2019bert}, ALBERT \cite{lan2020albert}, and RoBERTa \cite{liu2019roberta}) has proven to be one of the most promising paradigms for tackling Natural Language Processing (NLP) tasks. Many cutting edge NLP models achieve state-of-the-art results by fine-tuning pre-trained models. However, it has been observed by many researchers that existing fine-tuning procedures suffer from the instability problem \cite{devlin2019bert,phang2018sentence,lee2019mixout,zhu2020freelb,dodge2020fine,pruksachatkun2020intermediate,mosbach2020stability,zhang2020revisiting,zhao2021calibrate,han2021robust}, i.e., fine-tuning a model with the same setting results in significantly different performance. Such instability problem substantially impairs the model performance and makes different fine-tuned models incomparable with each other. Many different approaches have been proposed to solve this problem. \citet{mosbach2020stability} propose to use a smaller learning rate and more iteration steps, while \citet{arora2018stronger,sanyal2019stable,hua2021noise,aghajanyan2020better} propose to control the Lipschitz constant with different noise regularizations. However, there is no unified theoretical framework to help understand the effectiveness of these methods.

\begin{table*}[t]
    \centering
    \scriptsize
    
    \begin{tabular}{llll}
    \toprule
    \textbf{\ \ \ \ \ \ \ \ \ \ \  Methods } &    \textbf{\ \ \ Theoretical Basis} &  \textbf{Our Experimental Verification} &  \textbf{\ \ \ \ \ Reference}   \\
    \hline
    Increase sample size $n$  & \Cref{thm:propgd}, \Cref{thm:mainthm} & \Cref{sec:datasize-std}, \Cref{sec:N-pertube} & \citet{devlin2019bert} \\
    Decrease Lipschitz constant $L$  & \Cref{thm:propgd}, \Cref{thm:thmlnsr} & \Cref{sec:expmain}, \Cref{sec:headtune}, \Cref{sec:datastablity} &\citet{hua2021noise} \\
    Increase iteration steps $T$  & \Cref{thm:mainthm} & \Cref{sec:epoch-std} & \citet{mosbach2020stability} \\
    Use smaller learning rate $\eta$  & \Cref{thm:mainthm} & \Cref{sec:lr-std} &  \citet{mosbach2020stability} \\
    Max Margin Regularizer (MMR) & \Cref{thm:mainthm} & \Cref{sec:expmain}, \Cref{sec:headtune}, \Cref{sec:datastablity}, \Cref{sec:margin-pertube} & \textbf{Our New Method} (\Cref{sec:mmr}) \\
    Multi-Head Loss (MHLoss)  &\Cref{thm:thmmhl} & \Cref{sec:expmain}, \Cref{sec:head-std}, \Cref{sec:headtune}, \Cref{sec:datastablity}, \Cref{sec:head-coverage} & \textbf{Our New Method} (\Cref{sec:mhloss}) \\
    Self Unsupervised Re-Training (SURT)  & \Cref{thm:propgd} & \Cref{sec:expmain}, \Cref{sec:headtune}, \Cref{sec:datastablity}, \Cref{sec:dist-stable}& \textbf{Our New Method} (\Cref{sec:surt})\\
    \bottomrule
    \end{tabular}
    \caption{Methods for stabilizing fine-tuning. We provide a novel theoretical analysis of existing methods. Based on our theory, we also propose three new methods to stabilize the fine-tuning procedure. We conduct extensive experiments to verify our theoretical observations while more detailed empirical verifications of existing methods can also be found in the corresponding reference.}
    \label{tab:compare}

\end{table*}

In this paper, we give a unified theoretical stability analysis of two most widely used fine-tuning paradigms, namely, the full fine-tuning \cite{devlin2019bert} and the head tuning \cite{Peters2018DeepCW,wei2021pretrained}  (also called linear probing by \citet{peters2019tune,chen2021empirical,kumar2021fine}). Full fine-tuning means tuning all the parameters initialized with the pre-trained encoder while head tuning means freezing the pre-trained encoder and only tuning the specific task head layer \cite{kumar2021fine} on top of that encoder. Different from training from scratch, a fine-tuned pre-trained model naturally possesses many good properties and is amenable to theoretical analysis. Specifically, as empirically indicated by \citet{radiya2020fine}, the pre-trained parameters and the fine-tuned parameters are very close to each other. This observation motivates us to approximate the original function with its second-order Taylor expansion, which provides great convenience for theoretical analysis. To analyze the full fine-tuning, we first define the leave-one-out model stability following \citet{bousquet2002stability,schliserman2022stability} and then prove a stability upper bound by analyzing the model's second-order Taylor expansion. This bound explains why increasing the training sample size or reducing the Lipschitz constant can help stabilize the fine-tuning procedure. Moreover, following \citet{wei2021pretrained,kumar2021fine}, we further give a theoretical analysis of the head tuning paradigm where only a linear head is trained. Our theoretical analysis shows that increasing the iteration number, increasing the training sample size, or using a smaller learning rate stabilizes the fine-tuning procedure. This observation is also consistent with many empirical results \cite{mosbach2020stability,hua2021noise}. We list these widely used stabilizing methods with their corresponding theoretical basis in \Cref{tab:compare}. We also conduct comprehensive experiments to further verify these conclusions.

Our theoretical analysis can not only explain the principle behind majority of known empirical facts but also contribute to the design of novel techniques to help stabilize the fine-tuning procedure. These methods are also shown in \Cref{tab:compare}. 
First, we propose a novel Maximal Margin Regularizer (MMR) that maximizes the margin between the encoded features from different classes by adding a new regularization term to the original loss. Minimizing this term increases the distance between the features. We show both theoretically and empirically that increasing this margin can help improve fine-tuning stability. 
Afterwards, we propose a novel Multi-Head Loss (MHLoss), where we train several linear heads simultaneously and combine them together to predict the label. We theoretically prove that such a combination accelerates the convergence rate of the training procedure and thus improves the fine-tuning stability. 
Finally, we propose a novel Self Unsupervised Re-Training (SURT) method to show that fine-tuning on a model with weights closer to the final weights helps stabilize the fine-tuning procedure. We re-train the model with the masked language model task on the training data without labels and then fine-tune this model on the training data. This method originates from our theoretical prediction that reducing the distance between the pre-trained parameters and the fine-tuned parameters help improve the fine-tuning stability. 
We also conduct extensive experiments with our methods on both the full fine-tuning setting and the head tuning setting to demonstrate that these methods are empirically applicable to both settings.

Our contributions (also shown in \Cref{tab:compare}) can be summarized as follows: (1) We give a theoretical stability analysis of two most popular fine-tuning settings, namely, the full fine-tuning and the head tuning. (2) Our theory explains the effectiveness of many existing works that stabilize the fine-tuning. (3) We design three novel methods to stabilize the fine-tuning based on our theory. (4) We conduct extensive experiments on 11 real-world NLP datasets, as well as a bunch of synthetic classification tasks to show the effectiveness of our theory and our newly proposed methods.

\section{Theoretical Analysis}

\subsection{Notation}
We introduce the notation used throughout this paper. Here, we focus on the classification tasks to simplify our analysis, as most of the NLP tasks can be described as classification tasks. Let $S=\{(x_1,y_1),(x_2,y_2),\cdots,(x_n,y_n)\}$ be a training set, where $x_i\in \mathbb{R}^{d_x}$ is an input feature vector, $y\in \{-1,1\}$ is the corresponding output label, $n$ is the sample size, and $d_x$ is the dimention size for vector $x_i$. We denote the feature encoder as $E$, which is usually a pre-trained encoder. The augmented encoded representation of $x_i$ is calculated as $\tilde{x}_i=[E(x_i)^\T,-1]^\T$. We define $S^i=\{(x_1,y_1),(x_2,y_2),\cdots,$ $(x_{i-1},y_{i-1}),(x_{i+1},y_{i+1}),\cdots,(x_n,y_n)\}$ as a pertubation \cite{bousquet2002stability} of the training set $S$ by removing sample $(x_i,y_i)$. We denote the trainable parameter as $w=\A(S)\in \mathbb{R}^{d_w}$, which is obtained by training the dataset $S$ with algorithm $\A$. Here, $w$ is a column vector and $d_w$ represents the dimension of the parameter vector $w$. We denote the pre-trained initialization of the trainable model parameter as $w_0$, the parameter at the $t$th iteration step as $w_t$, and the optimal solution as $w_*$. We use $M\succeq 0$ (resp., $M \succ 0$) to indicate that the matrix $M$ is positive semidefinite (resp., positive definite). Let $f: \mathbb{R}^d \rightarrow \mathbb{R}$ be a function. We say that $f$ is convex if $f(t x+(1-t) y) \leq t f(x)+(1-t) f(y)$ for all $x,y \in \mathbb{R}^d$ and $ t \in[0,1]$; $f$ is $L$-Lipschitz if $\|f(x)-f(y)\|\le L \|x-y\| $ for all $x,y \in \mathbb{R}^d$; $f$ is $\beta$-smooth if $\|\nabla f(x)-\nabla f(y)\|\le \beta \|x-y\| $ for all $x,y \in \mathbb{R}^d$; $f$ is $\mu$-strong convexity if $
f(y) \geq f(x)+\langle\nabla f(x),(y-x)\rangle+\frac{\mu}{2}\|y-x\|_2^2$ for all $x,y \in \mathbb{R}^d$, where $\langle a, b \rangle$ represents the inner product of the vectors $a$ and $b$ while $\nabla f(x)$ is the gradient of $f$ at $x$. We use $M^\T$ to denote the transpose of matrix $M$.

\subsection{Stability Analysis for Full Fine-Tuning}\label{sec:fulltune}

Full fine-tuning tunes all parameters of a given pre-trained model. However, directly giving a theoretical analysis of such a complex function is quite challenging. Fortunately, based on the empirical observation by \citet{radiya2020fine} that the final fine-tuned weights are close to the pre-trained weights, we can approximate the function by its second-order Taylor expansion. Then, we apply \citet{schliserman2022stability}'s stability bound and show that it converges to a finite bound. This bound gives us a powerful tool to analyze many existing methods. Specifically, this bound theoretically explains the effectiveness of increasing training sample size \cite{devlin2019bert} and lowering the Lipschitz constant \cite{hua2021noise} in stabilizing the fine-tuning procedure. We will also use this theorem to help design our new SURT method.

Before we delve into the stability analysis, we should first give a formal definition of model stability. Stability theory for general learning algorithms has been extensively explored by many previous works \cite{bousquet2002stability,shalev2010learnability,charles2018stability}. They propose to use the pointwise hypothesis stability to measure the output variation after removing one of the training samples. On the other hand, \citet{schliserman2022stability} propose to directly use the distance between model parameters as a measure of stability. This leads to the following definition.

\begin{definition}
[Leave-one-out Model Stability \cite{schliserman2022stability}] We say that a learning algorithm $\A$ has leave-one-out model stability $\epsilon$ if  $\forall i\in \{1,\cdots,n\}$,
\begin{equation}
    \E_{S}\left[\left\|\A(S^{i})-\A(S)\right\|\right]\le \epsilon.
\end{equation}
\label{def:stability}
\end{definition}

With a slight abuse of notation, we denote $\epsilon$ as the infimum over all feasible $\epsilon$'s for which \Cref{def:stability} holds. To analyze the behavior of training the model on the dataset $S$, we first assume that the overall loss function $f$ is $L$-Lipschitz and $\beta$-smooth. These two assumptions are widely used in the analysis of the behavior of neural networks \cite{shalev2014understanding,nesterov2018lectures,schliserman2022stability}. Moreover, as empirically indicated by \citet{radiya2020fine}, the pre-trained parameter $w_0$ and the fine-tuned parameter $w_*$ are very close to each other. Therefore, around the optimal solution $w_*$, $f(w,x)$ can be approximated by its second-order Taylor expansion as
\begin{equation}
    \begin{aligned}
        f(w,x)=&f(w_*,x)+(w-w_*)^{\T}\nabla f(w_*,x)+\frac{1}{2}(w-w_*)^{\T}\nabla^2f(w_*,x)(w-w_*),
    \end{aligned}
\end{equation}

where $x$ stands for given fixed input data while $w$ is the function parameter, Since $w_*$ is the optimal solution, the Hessian matrix $\nabla^2f(w_*,x)$ is positive semidefinite. To simplify our analysis, we focus on the scenario where the Hessian matrix is positive definite and satisfies $ \beta I \succeq \nabla^2f(w_*,x) \succeq \mu I $ with $\mu>0$. Recall that the gradient descent method iterates as $w_{t+1}=w_t-\eta \frac 1 n \sum_{i=1}^n \nabla f(w_t,x_i)$, where $w_{t+1}$ denotes the weight at the $(t+1)$st iteration and $\eta$ is the learning rate. Moreover, as indicated in \citet{nakkiran2021deep,belkin2019reconciling,ishida2020we}, big pre-trained models almost achieve zero training loss. Hence, we assume that $f(w_*,x)=0$. Now, we are ready to present the following stability bound for the Taylor expansion of full fine-tuning.

\begin{restatable}[Stability Bound for Full Fine-Tuning]{theorem}{propgd}
\label{thm:propgd}
Suppose that the loss function $(w,x) \mapsto f(w,x)$ is non-negative, $L$-Lipschitz, and $\beta$-smooth with respect to $w$, $\mu I\preceq \nabla^2f(w_*,x)$ with $\mu>0$, and $f(w_*,x)=0$. If $\mathcal{A}$ is the gradient descent method with learning rate $\eta =\frac{1}{\beta}$, then the leave-one-out model stability satisfies 
    \begin{equation}
        \begin{aligned}
            \E_{S}\left[\left\|\A(S^{i})-\A(S)\right\|\right]&\le \frac{ \sqrt{2L\|w_0-w_*\|/\beta}}{n(1/\sqrt[4]{1-\frac \mu \beta}-1)} .
        \end{aligned}
        \label{eqn:fftbound}
    \end{equation}

\end{restatable}

The proof can be found in \Cref{sec:A-propgd}. It can be observed from \Cref{thm:propgd} that increasing the training sample size $n$ can reduce the term on the right-hand side of \Cref{eqn:fftbound}. Therefore, it brings down the stability upper bound and potentially stabilizes the fine-tuning procedure. This insight is consistent with the empirical findings from numerous earlier papers \cite{devlin2019bert,lee2019mixout,dodge2020fine}. In addition, \citet{cattan2022usability} observe that data augmentation can strengthen the stability, which further supports the theoretical conclusion. The intuition behind this prediction is that as the sample size $n$ increases, the impact of the noise introduced by the training set perturbation is mitigated and thus the stability improves. We can also conclude from \Cref{thm:propgd} that reducing the Lipschitz constant $L$ for the function $f$ can similarly diminish the leave-one-out model stability, hence stablizing the training procedure. This phenomenon has also been examined by numerous recent works \cite{arora2018stronger,sanyal2019stable,hua2021noise,aghajanyan2020better}, which propose to impose a noise perturbation component on the input features and then minimize the distance between the original output and the noise output. Our theoretical analysis explains why controlling the Lipschitz constant might enhance stability. Lastly, we note that reducing the distance $\|w_0-w_*\|$ between the initial parameter $w_0$ and the optimal parameter $w_*$ can also improve stability. Intuitively, if the start point and the endpoint are close to each other, the optimization procedure is less likely to jump to some other local minima, thus making the training procedure more stable. We will propose a novel SURT (\Cref{sec:surt}) method based on this observation.

\subsection{Stability Analysis for Head Tuning}\label{sec:headtuneanalysis}

In \Cref{thm:propgd}, we assume some properties of the loss function $f$ and study how the sample size $n$, Lipschitz constant $L$, and parameter distance $\|w_0-w_*\|$ control the stability. In this section, we further unveil more influencing factors under the head tuning setting. The intuition is also very straightforward. According to \citet{soudry2018implicit}, when training a linear classifier with gradient descent, it tends to be a maximal margin classifier. Based on this conclusion, we use the telescope sum to bound the stability by the distance between the current iterate's weights and the final maximal margin classifier's weights together with the distance between the original and perturbed maximal margin classifiers. This bound theoretically explains why increasing training sample size \cite{devlin2019bert}, increasing iterations number \cite{mosbach2020stability}, or using a smaller learning rate \cite{mosbach2020stability} can help stabilize the fine-tuning procedure. We also derive a new corollary to show that limiting the Lipschitz constant \cite{hua2021noise} can help stabilize the fine-tuning procedure under this scenario. We will further propose our new MMR and MHLoss method based on the same theory.

Formally, similar to \Cref {sec:fulltune}, following the observation of \citet{radiya2020fine}, the parameters of the pre-trained model and the final results are very close to each other. We can assume that the parameters of the encoder $E$ are fixed during fine-tuning. This assumption results in a very popular fine-tuning method called head tuning \cite{Peters2018DeepCW,wei2021pretrained}, where the encoder parameters are frozen and only the linear head is tunable. 
Moreover, following \citet{soudry2018implicit,schliserman2022stability}, we assume that the encoded features are linearly separable. We refer readers to the detailed discussion in \citet{ji2019implicit} for an analysis of non-separable data. Under the above setting, given a dataset $S$, each sample $x_i\in S$ is first encoded as $\tilde{x}_i=[E(x_i)^\T,-1]^\T$ and then $\tilde{x}_i$ is classified by a linear head layer. Following \citet{soudry2018implicit}, we define the head tuning loss as $\ell(w^{\T}\tilde{x}_iy_i)$, where $w=[v^\T,b]^T$ and $\ell(u)$ is an arbitrary non-negative, differentiable, $\beta$-smooth, and monotonically decreasing to zero function. The linear bias term has already been merged into the weight $w$ \cite{soudry2018implicit}. Therefore, the overall loss function can be written as
\begin{equation}\mathcal{L}(E,w)=\frac 1 n \sum_{i=1}^n \ell(w^{\T}\tilde{x}_iy_i).\end{equation}
As proved in Lemma 1 of \citet{soudry2018implicit}, if we train a linear model on a linearly separable dataset with the loss $\ell$, the norm of $w$ must diverge toward infinity as $\lim _{t \rightarrow \infty}\|w_t\|=\infty$. Therefore, calculating the leave-one-out model stability in \Cref{def:stability} is not feasible as the parameters $\A(S)$ and $\A(S^i)$ both diverge toward infinity. Fortunately, under this circumstance, only the direction of the predictor, namely the normalized weight $w/\|w\|$, is important. As a result, we define a new stability measure called normalized leave-one-out model stability focusing on the discrepancy of the normalized weight $w/\|w\|$.

\begin{definition}
[Normalized Leave-one-out Model Stability] We say that a learning algorithm $\A$ has normalized leave-one-out model stability $\epsilon$ if $\forall i\in \{1,\cdots,n\}$,
\begin{equation}
    \E_{S}\left[\left\|\frac{\A(S^{i})}{\|\A(S^{i})\|}-\frac{\A(S)}{\|\A(S)\|}\right\|\right]\le \epsilon.
\end{equation}

\label{def:nstability}
\end{definition}

Different from \Cref{def:stability},  \Cref{def:nstability} normalizes the parameters trained on the corresponding training data and focuses on the direction gap between them. This definition is more reasonable for analyzing the tunable linear head as it also works even if $\lim _{t \rightarrow \infty}\|w_t\|=\infty$. 
To facilitate a theoretical analysis of head tuning, we further denote $\tilde{X}=[\tilde{x}_1,\cdots,\tilde{x}_n]^{\T}\in \mathbb{R}^{d_x \times n}$ and denote $\sigma_{\max}(\tilde{X})$ as the largest singular value of $\tilde{X}$. The head tuning approach aims to find a separation plan $w^{\T}\tilde{x}=0$ to classify the encoded features $\tilde{x}_i$ into two categories. Here, $w=[v^\T,b]^\T=\A(S)$ is the classifier parameter. We denote $\hat{w}_{S}=[\hv_S^\T,\hb_S]^\T$ as the SVM solution trained on dataset $S$ and denote $\gamma_S$ as the maximal margin between separation plans and encoded features, which can be calculated as $\gamma_S=\frac{1}{\|\hat{v}_S\|}$ \cite{bishop2006pattern}. Similarly, we denote $\hat{w}_{S^i}=[\hv_{S^i}^\T,\hb_{S^i}]^\T$ as the SVM solution trained on the dataset $S^i$. Here, $\hv_S^\T,\hv_{S^i}^\T\in \mathbb{R}^{d_x}$ are the weights while $\hb_{S},\hb_{S^i}$ are the intercepts. Then, we present the theorem for head tuning as follows.

\begin{restatable}[Stability Bound for Head Tuning]{theorem}{mainthm}
\label{thm:mainthm}
Given a linearly separable dataset $S$, suppose that the encoded features $E(x_i)$ are bounded as $\|E(x_i)\|\le B$, $\forall i \in \{1,\cdots,n\}$. Let $\gamma_S$ be the maximal margin between separation plan $\hw_S^{\T}\tilde{x}=0$ and encoded features $E(x_i)$. Suppose further that the model parameter $w$ is optimized by gradient descent with $t$ iterations and learning rate $\eta< 2\beta^{-1}\sigma_{\max}^{-1}(\tilde{X})$. Then, for some constants $C, \lambda, \nu$, the normalized leave-one-out model stability is upper bounded as
\begin{equation}
    \E_{S}\left[\left\|\frac{\A(S^{i})}{\|\A(S^{i})\|}-\frac{\A(S)}{\|\A(S)\|}\right\|\right]\le \frac{C\log \log t}{\log t}+\nu \max\big\{\sqrt{\frac{2}{\lambda n}\left(1+\frac {B} {\gamma_S}\right)}, \frac{B+\sqrt{B^2+8n\lambda(1+B/\gamma_S)}}{2n\lambda} \big\}.
\label{eqn:maineqn}
\end{equation}
\end{restatable}

The proof can be found in \Cref{sec:A-mainthm}. This theorem is based on \citet{soudry2018implicit}'s theory that training a linear model on linearly separable data will converge to the direction of a max-margin classifier (SVM) if it is trained with the gradient descent method. To give a whole picture of analyzing the whole procedure, we use the telescope sum to incorporate the gap between two max-margin classifiers when trained on different datasets ($S$ and $S^i$), as well as the gap between parameter $w_t$ and the corresponding max-margin classifier $\hat{w}_S$. Specifically, the first term ($\frac{C\log \log t}{\log t}$) in \Cref{eqn:maineqn} indicates how the parameter $w_t$'s direction converges to the corresponding max-margin classifier $\hat{w}_S$'s direction at the $t$th step. The second term $\nu \max\big\{\sqrt{\frac{2}{\lambda n}\left(1+\frac {B} {\gamma_S}\right)}, \frac{B+\sqrt{B^2+8n\lambda(1+B/\gamma_S)}}{2n\lambda} \big\}$ measures the direction discrepancy between the two max-margin classifiers trained with the datasets $S$ and $S^i$. It can be observed from \Cref{thm:mainthm} that increasing the number of iterations $t$ can help stabilize the training procedure. This phenomenon has already been empirically observed by \citet{mosbach2020stability} in extensive experiments. Our theory further gives the intuition and theoretical understanding of why increasing the iteration number stabilizes the training procedure. The first term in \Cref{thm:mainthm} indicates that, as the training step increases, the model parameter $w_t$'s direction is closer to the corresponding max-margin classifier $\hat{w}_S$'s direction. Unfortunately, the rate $\mathcal{O}(\frac{C\log \log t}{\log t})$ converges relatively slowly \cite{soudry2018implicit}. We will derive a new corollary based on \Cref{thm:mainthm} and design a novel multi-head loss to accelerate the convergence rate in \Cref{sec:mhloss}. In addition, increasing the sample size $n$ can help stabilize the model. This observation is the same as that observed in \Cref{thm:propgd}. It shows that this method is theoretically effective under both settings. Moreover, if the encoded representation $E(x_i)$ has large margin $\gamma_S$, the model will be more stable. This observation is also very intuitive. If the margin is large, it becomes easier to find a separation plane to separate the data, while a small perturbation of the plane can hardly interfere with the classification results. We will propose a novel max margin regularizer (\Cref{sec:mmr}) based on this observation. Lastly, using a smaller learning rate is a necessary condition for \Cref{thm:mainthm} to be held. This is because only if the stepsize is small enough, the weight can be guaranteed to converge to the max-margin classifier weight. It should also be noted that the widely used Descent Lemma (\Cref{lemma:descent}) implies that a smaller learning rate helps improve the stability, as it is a sufficient condition to guarantee that the model has sufficient descent for each iteration step. It prevents the parameter from jumping to other local minima, making it more likely to converge to the same local minimum.

We further derive a corollary of \Cref{thm:mainthm} to show how the Lipschitz constant controls the stability in the head tuning setting. We simplify the distance between parameters trained on two datasets to incorporate the Lipschitz constant $L$.
\begin{restatable}{corollary}{thmlnsr}
\label{thm:thmlnsr}
Given a linearly separable dataset $S$, suppose that the encoded features $E(x_i)$ are bounded as $\|E(x_i)\|\le B$, $\forall i \in \{1,\cdots,n\}$. Suppose further that the model parameter $w$ is optimized by gradient descent with $t$ iterations and learning rate $\eta< 2\beta^{-1}\sigma_{\max}^{-1}(\tilde{X})$. For some constants $C, \lambda, \nu$, the normalized leave-one-out model stability is upper bounded as
\begin{equation}\E_{S}\left[\left\|\frac{\A(S^{i})}{\|\A(S^{i})\|}-\frac{\A(S)}{\|\A(S)\|}\right\|\right]\le \frac{C\log \log t}{\log t}+\nu \frac{L }{\lambda n}.\end{equation}
\end{restatable}

The proof can be found in \Cref{sec:A-thmlnsr}. It can be observed from \Cref{thm:thmlnsr} that if the function $f$ has smaller Lipschitz consistent $L$, training a model can be more stable. Intuitively, if the Lipschitz constant is small, it becomes less sensitive to data perturbation. As a result, the directions for the max-margin classifiers will be quite close to each other, which increases the stability. This conclusion is consistent with \Cref{thm:propgd} and has also been empirically examined by many previous works \cite{arora2018stronger,sanyal2019stable,hua2021noise,aghajanyan2020better} with extensive experiments. As a result, this strategy, in conjunction with the discussion in \Cref{sec:fulltune}, is theoretically justified in both full fine-tuning and head tuning settings.

\section{Methods}
In addition to being able to explain most of the observed empirical discoveries, our theory can also help in the design of effective and provable methods. Armed with our new theory, we propose three novel methods to stabilize the fine-tuning procedure, which further verify the correctness of our theory. First, based on \Cref{thm:mainthm}, we propose a Max Margin Regularizer (MMR) to maximize the representation margin between samples from different classes. Then, we prove \Cref{thm:thmmhl} based on \Cref{thm:mainthm}, which establishes the theoretical basis for our novel Multi-Head Loss (MHLoss). It utilizes a multi-head layer to accelerate the convergence rate and thus stabilize the fine-tuning procedure. Finally, based on \Cref{thm:propgd}, we propose a Self Unsupervised Re-Training (SURT) method to initiate fine-tuning from a point closer to $w_*$. We will conduct extensive experiments in \Cref{sec:exp} to verify our proposed methods.

\subsection{Max Margin Regularizer}\label{sec:mmr}
It can be observed from \Cref{thm:mainthm} that if the margin $\gamma_S$ between the encoded representation and the separation plane is large, the model will have better stability. Based on this intuition, we propose a novel Max Margin Regularizer (MMR) for fine-tuning to help maximize the margin. However, calculating the margin is quite computationally costly and the margin itself is not differentiable. To tackle this problem, we propose to maximize the distance between the center points of the two classes. Intuitively, if the distance between the class centers increases, the margin between the encoded representation and the separation plane will also likely to increase. We recall that given a training set $S$, the input $x_i\in S$ will be first encoded with the encoder $E$ as $E(x_i)$. Each category should contain at least one sample and MMR can then be represented as

\begin{equation}
    \mathcal{R}(S)=\frac{1}{1+\left \|\sum_{i=1}^n E(x_i)y_i\left ( \frac{1+y_i}{\sum_{j=1}^n (1+y_j)} + \frac{1-y_i}{\sum_{j=1}^n (1-y_j)} \right)\right \|}.
\end{equation}

Intuitively, the above calculates the center points for different classes and then gets the distance between them. We use the reciprocal of the distance as the regularization term to ensure that minimizing the regularization term will result in increased distance. We add a constant 1 on the denominator to increase numerical stability. Therefore, the final optimization target can be written as \begin{equation}\mathcal{L}_{\text{MMR}}(E,w)=\frac 1 n \sum_{i=1}^n \ell(w^{\T}\tilde{x}_iy_i)+\alpha \mathcal{R}(S),\end{equation}
where $\alpha$ is the weight parameter for $\mathcal{R}(S)$.

\subsection{Multi-Head Loss}\label{sec:mhloss}
As indicated in \citet{soudry2018implicit}, the convergence rate $\mathcal{O}(\frac{\log \log t}{\log t})$ for the first term in \Cref{eqn:maineqn} is quite slow as $t$ grows. As a result, the effect of raising $t$ to lower the bound gradually loses its apparent effect especially when $t$ is already very large. To further reduce this term, we propose a novel Multi-Head Loss (MHLoss). Specifically, instead of using one linear head to calculate the loss, we propose to use $H$ linear headers with the same shape simultaneously to calculate the loss and take the average of the outputs. In the training stage, the $h$th head ($h\in \{1,\cdots,H\}$) with parameter $w_h$ is trained separately by minimizing the loss $\ell(w^{\T}_h\tilde{x}_iy_i)$. The overall loss can be calculated as \begin{equation}\mathcal{L}_{\text{MH}}(E,w_1,\cdots,w_H)=\frac 1 {nH} \sum_{h=1}^H\sum_{i=1}^n \ell(w_h^{\T}\tilde{x}_iy_i).\end{equation}
In the testing stage, we can calculate the result for an input $x$ by averaging all the heads as $\frac 1 H \sum_{h=1}^H (w_h^{\T}\tilde{x}_i)=(\frac 1 H \sum_{h=1}^H w_h^{\T})\tilde{x}_i=\bar{w}^{\T}\tilde{x}_i$, where $\bar{w}^{\T}$ is the mean average of all $w_h$'s. It is interesting to note that the final model is a combination of several linear models, which is still a linear model without any extra structure added. We argue that this loss helps improve the stability, because it accelerates the convergence speed for the first term in \Cref{eqn:maineqn}. To theoretically prove this claim, we establish \Cref{thm:thmmhl}, which is based on \Cref{thm:mainthm}. Here, as indicated by \citet{soudry2018implicit} that all the weights $w_h$'s converge to the same direction as $\hw_S$, we mainly focus on the case where $\bar{w}$ is not orthogonal to $\hw_S$ during the training procedure, namely, $\bar{w}^\T \hw_S\ne 0$. It shows that this combination accelerates the converging speed and improves the stability if the training step is fixed.

\begin{restatable}[Stability Bound for Multi-Head Loss]{corollary}{thmmhl}
\label{thm:thmmhl}
Consider a mulit-head loss with $H$ heads, where $H>2+8\ln\frac 1 \delta$, $\delta\in (0, 1)$, and $\bar{w}^\T \hw_S\ne 0$. With the same assumptions as in \Cref{thm:mainthm}, for some constants $C, \xi, \nu$, with probability $1-\delta$, we have 
\begin{equation}
\resizebox{0.94\textwidth}{!}{
        $
        \E_{S}\left[\left\|\frac{\A(S^{i})}{\|\A(S^{i})\|}-\frac{\A(S)}{\|\A(S)\|}\right\|\right] \le \sqrt{\frac{2 +8\log \frac 1 \delta}{H}}\frac{C\xi\log \log t}{\log t}+
        \nu \max\big\{\sqrt{\frac{2}{\lambda n}\left(1+\frac {B} {\gamma_S}\right)}, \frac{B+\sqrt{B^2+8n\lambda(1+B/\gamma_S)}}{2n\lambda} \big\}.\label{eqn:mhloss}
        $%
        }
\end{equation}
\end{restatable}

The proof can be found in \Cref{sec:A-thmmhl}. It can be observed from \Cref{thm:thmmhl} that the stability is bounded by a term that involves the head number $H$. As $H$ increases, the first term in \Cref{eqn:mhloss} decreases at the rate of $\mathcal{O}(\frac {1}{\sqrt{H}})$, which is better than simply using one head. The intuition behind the multi-head loss is also very straightforward. \Cref{lemma:exp0} shows that the expectation of the head parameter $w_h$ is SVM. As implied by the concentration property, if we take the average of these classifiers, we can get an averaged linear classifier closer to the max-margin classifier.

\subsection{Self Unsupervised Re-Training}\label{sec:surt} 
As discussed in \Cref{thm:propgd}, reducing the distance $\|w_0-w_*\|$ between the initialization weight $w_0$ and the solution weight $w_*$ reduces the stability upper bound. This observation inspires us to fine-tune a pre-trained model that is very close to the final model. To get such a pre-trained model, a straightforward idea is to utilize a model pre-trained on the same domain as the backbone model, because the feature encoder may have already been well-adapted to the specific domain and will not change too much to adapt to that domain during fine-tuning. Unfortunately, it is not possible for us to get such a well pre-trained model for an arbitrary domain. To solve this problem, we propose a novel Self Unsupervised Re-Training (SURT) method to first re-train the given pre-trained model with the same training corpus as the one used in the fine-tuning task. It is re-trained with the unsupervised mask language model \cite{devlin2019bert} task, which only needs the given training corpus without any annotated label. Then, we fine-tune the model based on the re-trained model with the given labeled data. It should be noted that many previous works have proposed domain-adaptive pre-training \cite{gururangan2020don,aghajanyan2021muppet,hendrycks2019using}. They re-train the model with an extra domain-specific corpus that is not always guaranteed to exist. Different from these models, our proposed SURT method directly re-trains the model with the original training set without the labels. It does not require finding extra corpus and is thus applicable to more domains. Also, to the best of our knowledge, our theoretical analysis is the first to show why using a re-trained model helps stabilize fine-tuning.

\section{Experiments}\label{sec:exp}

\subsection{Setup}
We evaluate our theoretical conclusions and newly proposed methods on both real-world GLUE/SuperGLUE datasets and synthetic classification datasets. Specifically, in \Cref{sec:mainexp}, we evaluate the methods with the widely used NLP benchmark datasets GLUE \cite{wang2018glue} and SuperGLUE \cite{wang2019superglue}. They contain several tasks including Corpus of Linguistic Acceptability (CoLA) \cite{warstadt2019neural}, Microsoft Research Paraphrase Corpus (MRPC) \cite{dolan2005automatically}, Recognizing Textual Entailment (RTE) \cite{dagan2005pascal,giampiccolo2007third,bentivogli2009fifth}, Commitment Bank (CB) \cite{de2019commitmentbank}, Winograd Schema Challenge (WSC) \cite{levesque2012winograd}, Quora Question Pairs (QQP) \cite{wang2018glue}, Stanford Sentiment Treebank (SST-2) \cite{socher2013recursive}, Winograd NLI (WNLI) \cite{levesque2012winograd}, BoolQ (Boolean Questions) \cite{clark2019boolq},  Multi-Sentence Reading Comprehension (MultiRC) \cite{khashabi2018looking} and Word-in-Context (WiC) \cite{pilehvar2019wic}. We follow the setting of many previous works \cite{phang2018sentence,lee2019mixout,dodge2020fine,xu2021raise} to use the original development set as our testing set since the original test set is only accessible via the submission link which contains a hard limit of the submission number. Moreover, we follow \citet{fu2022effectiveness} to split 90\% data from the original training set to serve as the new training set and we use the remaining 10\% samples as the new development set to tune model hyper-parameters like regularizer's weight, early stop epoch, etc. We use the early stop strategy to stop training if the performance on the new development set has no improvement for 20 consecutive evaluation epochs. We set $H=50$ for MHLoss in the main experiment (\Cref{sec:expmain}) and report the impact of using different $H$ in \Cref{sec:head-std}. The experiments are evaluated following the setting of \citet{wang2018glue,wang2019superglue}. In order to check the stability of the experiments, we run each experiment for 10 runs with different random seeds and report the mean scores and the standard deviations. Although our proposed methods are theoretically supported under different scenarios (full tuning and head tuning), we test them experimentally on both settings to show their capabilities of handling more scenarios. All the code is implemented based on the jiant framework \cite{phang2020jiant} with the RoBERTa-base \cite{liu2019roberta} model as the backbone model which is provided by the transformers\footnote{\url{https://huggingface.co/docs/transformers/model_doc/roberta}} toolkit \cite{wolf-etal-2020-transformers}. All the experiments are running on a cluster with NVIDIA A100 GPU cards. On the other hand, in \Cref{sec:simplecls}, we also conduct several experiments on synthetic classification tasks to show more details of how each factor affects the results. As discussed in \Cref{sec:headtuneanalysis}, the head tuning paradigm only tunes a linear head classifier on top of the embedded features. To give a clearer picture of how the methods work, we randomly generate several synthetic classification datasets containing features with respect to different requirements and validate some of our conclusions with them. These experiments can be used to show more details of the head tuning behaviors.

We compare our model with several widely used baseline models focusing on fine-tuning stability.
\textbf{FineTune} model is the standard full fine-tuning model \cite{devlin2019bert} that directly fine-tunes all the parameters of the backbone model together with the task-specific head on each task. 
\textbf{MixOut} model proposed by \citet{lee2019mixout} is another widely used fine-tuning method. It mixes the trained model parameters and the original model parameters according to a specific ratio.
\textbf{LNSR} model proposed by \citet{hua2021noise} uses a regularizer to diminish the Lipschitz constant of the prediction function to improve the stability. 

\subsection{Experiment Results for GLUE/SuperGLUE}
\label{sec:mainexp}

\begin{table*}[t]
    \centering
    \scriptsize
    \setlength{\tabcolsep}{1pt} 
    \resizebox{1.\textwidth}{!}{
        \scriptsize 
    \begin{tabular}{l|lllllllllll|l}
    \toprule
    {} &                             \multicolumn{1}{c}{CoLA} &                          \multicolumn{1}{c}{MRPC} &                              \multicolumn{1}{c}{RTE} &                            \multicolumn{1}{c}{CB} &                     \multicolumn{1}{c}{WSC} &                           \multicolumn{1}{c}{QQP} &                               \multicolumn{1}{c}{SST-2} &                                \multicolumn{1}{c}{WNLI} &                            \multicolumn{1}{c}{BoolQ} &                          \multicolumn{1}{c}{MultiRC} &                     \multicolumn{1}{c|}{WiC} &                                 \multicolumn{1}{c}{AVG} \\
    \hline
    FineTune &                                    59.21±2.32 &                              88.84±0.97 &                     \underline{76.21}±2.04 &                                 79.11±2.77 &            \underline{54.33}±5.48 &                                 39.67±8.13 &                                    93.04±0.48 &                                    54.65±3.44 &                              78.70±1.10 &                                    44.22±0.85 &                        68.07±1.19 &                                 66.91±2.61 \\
    MixOut   &                          \CCG{17.6}59.85±1.94 &        \CCG{16.2}\underline{89.42}±0.66 &                       \CCG{20.4}75.63±1.52 &                       \CCR{10.4}78.57±3.29 &              \CCR{28.0}51.92±6.88 &                       \CCG{10.0}39.67±8.13 &                          \CCG{11.2}93.13±0.42 &                          \CCG{29.8}54.75±2.45 &                    \CCG{16.4}78.54±0.78 &                 \CCG{10.6}\textbf{44.81}±0.82 &  \CCG{11.4}\underline{68.67}±1.12 &                       \CCG{11.2}66.82±2.55 \\
    LNSR     &  \CCG{23.4}\underline{60.69}±\underline{1.65} &        \CCG{16.4}89.21±\underline{0.65} &                       \CCG{31.4}74.26±0.97 &  \CCG{10.6}\textbf{80.71}±\underline{2.74} &              \CCG{26.0}52.79±4.68 &                       \CCG{11.2}39.61±8.07 &                 \CCG{12.0}\textbf{93.31}±0.38 &  \CCG{49.4}\underline{55.62}±\underline{1.47} &        \CCG{20.6}78.36±\underline{0.57} &                           \CCR{4.2}41.66±1.06 &      \CCR{3.0}\textbf{68.68}±1.34 &                       \CCG{19.4}66.81±2.14 \\
    MHLoss   &                           \CCR{1.0}59.64±2.37 &                    \CCG{13.8}88.43±0.78 &  \CCG{32.2}\textbf{77.11}±\underline{0.93} &                        \CCR{0.4}79.64±2.79 &     \CCG{24.6}\textbf{61.35}±4.75 &  \CCG{32.0}\underline{40.17}±\textbf{7.03} &                          \CCG{11.2}93.13±0.42 &        \CCG{53.4}\textbf{55.92}±\textbf{1.27} &  \CCG{21.4}\textbf{79.59}±\textbf{0.53} &                           \CCR{7.8}43.77±1.24 &  \CCG{14.6}67.51±\underline{0.96} &  \CCG{20.2}\textbf{67.84}±\underline{2.10} \\
    MMR      &                 \CCG{21.4}\textbf{60.96}±1.75 &  \CCG{17.6}\textbf{89.58}±\textbf{0.59} &              \CCG{39.8}75.59±\textbf{0.55} &           \CCR{20.8}\underline{80.00}±3.81 &  \CCG{48.2}54.13±\underline{3.57} &           \CCG{27.4}38.83±\underline{7.26} &     \CCG{13.4}\underline{93.18}±\textbf{0.31} &                          \CCG{32.6}54.79±2.31 &                     \CCR{2.6}76.96±1.23 &  \CCG{13.8}\underline{44.73}±\underline{0.66} &     \CCG{26.4}67.74±\textbf{0.37} &  \CCG{21.4}\underline{66.95}±\textbf{2.04} \\
    SURT     &                 \CCG{27.2}60.64±\textbf{1.46} &                    \CCG{11.2}89.35±0.91 &                       \CCG{30.8}76.03±1.00 &              \CCG{34.2}78.21±\textbf{1.56} &     \CCG{70.2}52.88±\textbf{2.47} &              \CCR{19.8}\textbf{40.42}±9.12 &  \CCG{12.4}\underline{93.18}±\underline{0.36} &                          \CCR{15.6}53.80±4.22 &        \CCG{15.8}\underline{78.86}±0.81 &                 \CCG{14.6}43.38±\textbf{0.62} &               \CCR{2.2}68.09±1.30 &                       \CCG{18.8}66.80±2.17 \\
    \bottomrule
    \end{tabular}
    }
    \caption{The main experiment. We mark the bests score with \textbf{bold} font and mark the second best scores with \underline{underline} mark. The gradient of \setlength{\fboxsep}{1pt}\colorbox{mina-green!30}{green} indicates the standard deviation improvement compared to the FineTune model (the deeper the more) while the gradient of \colorbox{red!30}{red} indicates the standard deviation degradation. The distribution of the scores is further thoroughly depicted in \Cref{sec:varbox} with a violin plot.}
    \label{tab:main}
\end{table*}

\begin{figure*}[t]
    \makebox[\linewidth][c]{
    \centering
    \begin{minipage}[t]{0.30\textwidth}
    \centering
    \includegraphics[width=0.99\columnwidth]{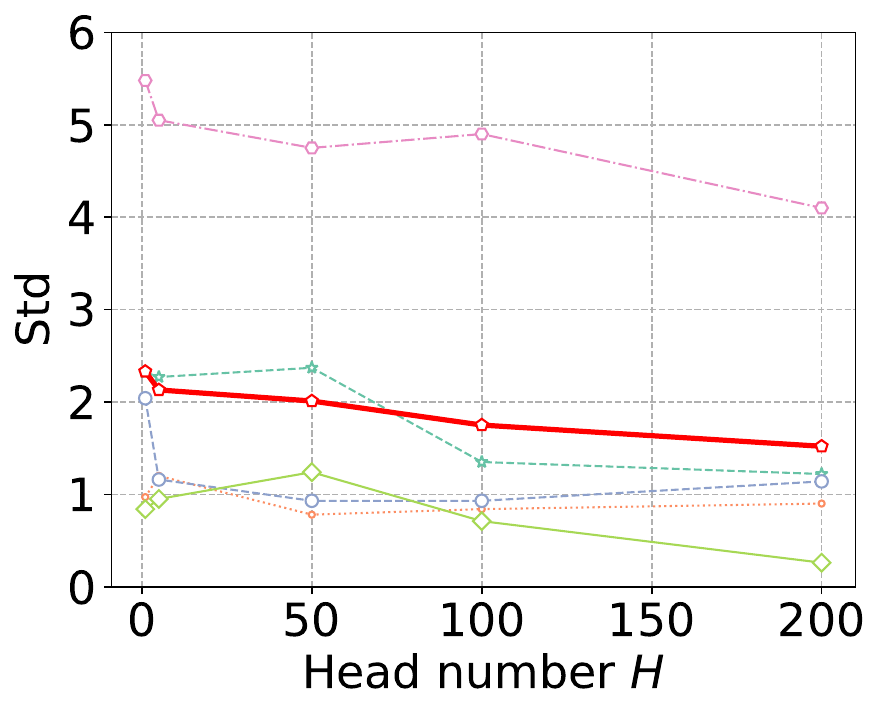}
    \caption{Impact of head number.}
    \label{fig:head-std}
    \end{minipage}
    \begin{minipage}[t]{0.301\textwidth}
        \centering
        \includegraphics[width=0.99\columnwidth]{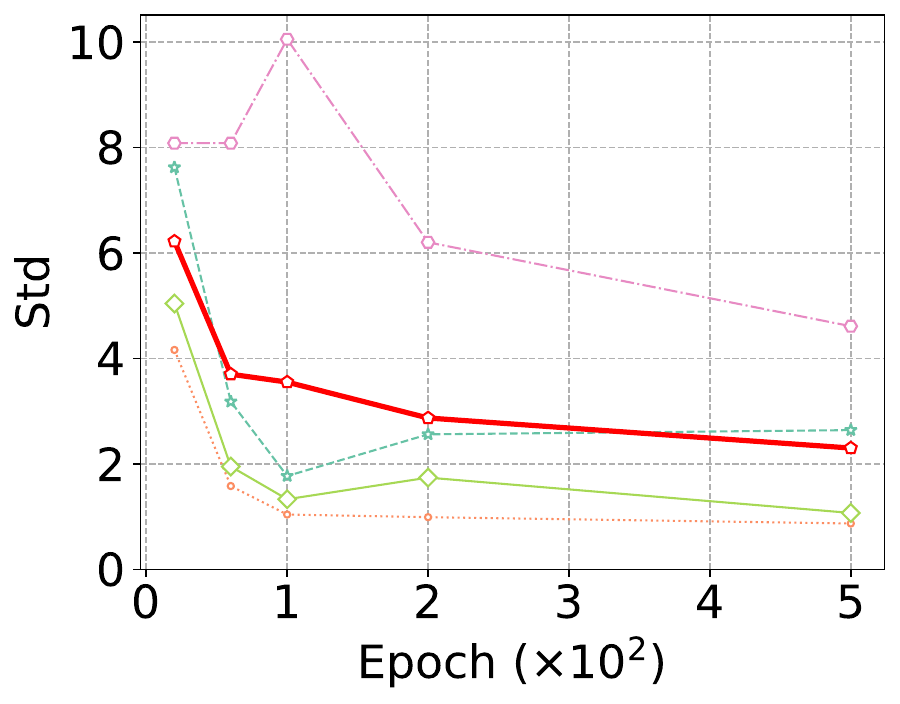}
    \caption{Impact of training epoch.}
    \label{fig:epoch-std}
    \end{minipage}
    \begin{minipage}[t]{0.42\textwidth}
        \centering
        \includegraphics[width=0.99\columnwidth]{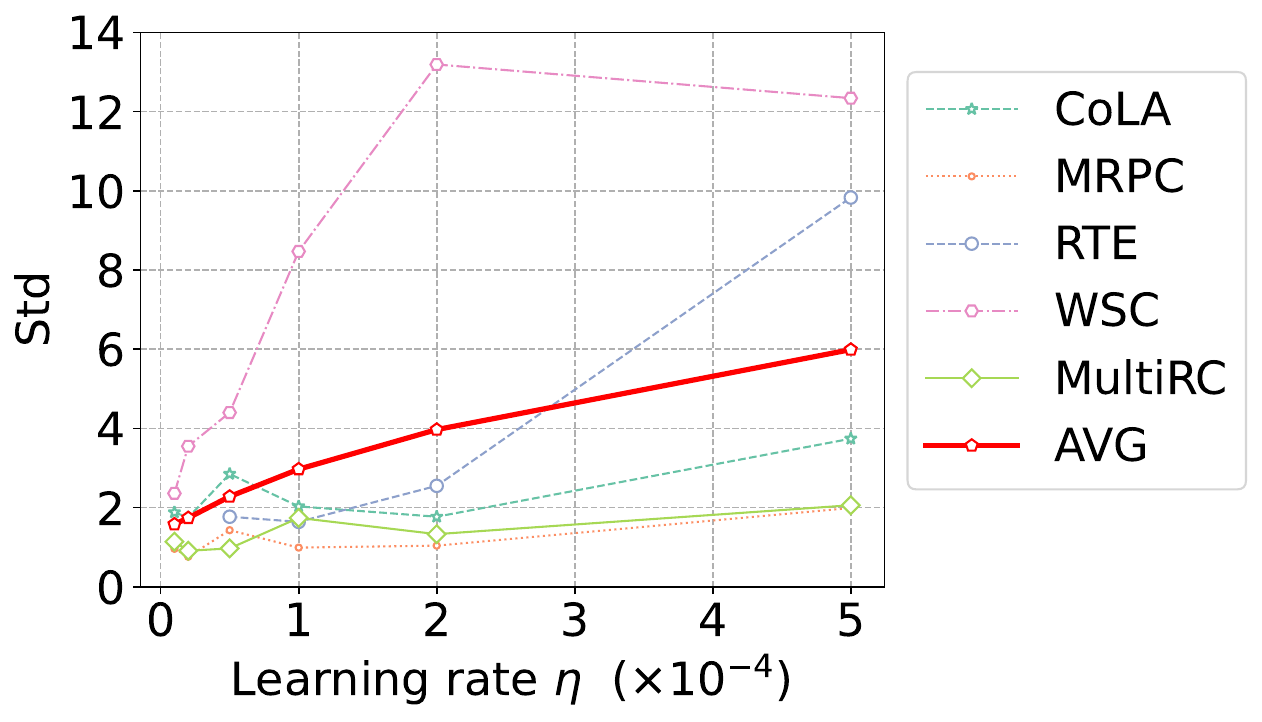}
        \caption{Impact of learning rate.}
        \label{fig:lr-std}
        \end{minipage}
    }
\end{figure*}

\begin{figure*}[t]
    \makebox[\linewidth][c]{
    \centering
    \begin{minipage}[t]{0.46\textwidth}
        \centering
        \includegraphics[width=0.99\columnwidth]{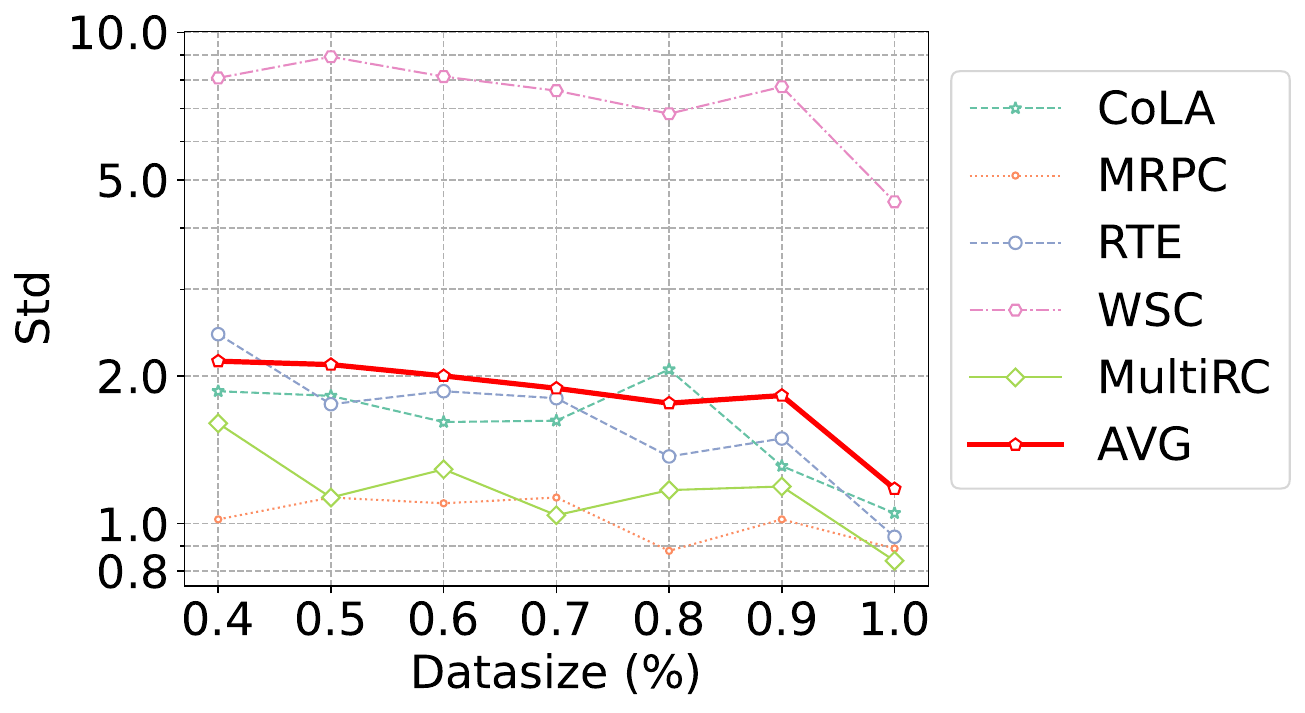}
    \caption{Influence of sample count.}
    \label{fig:datasize-std}
    \end{minipage}
    \begin{minipage}[t]{0.5\textwidth}
        \centering
        \includegraphics[width=0.99\columnwidth]{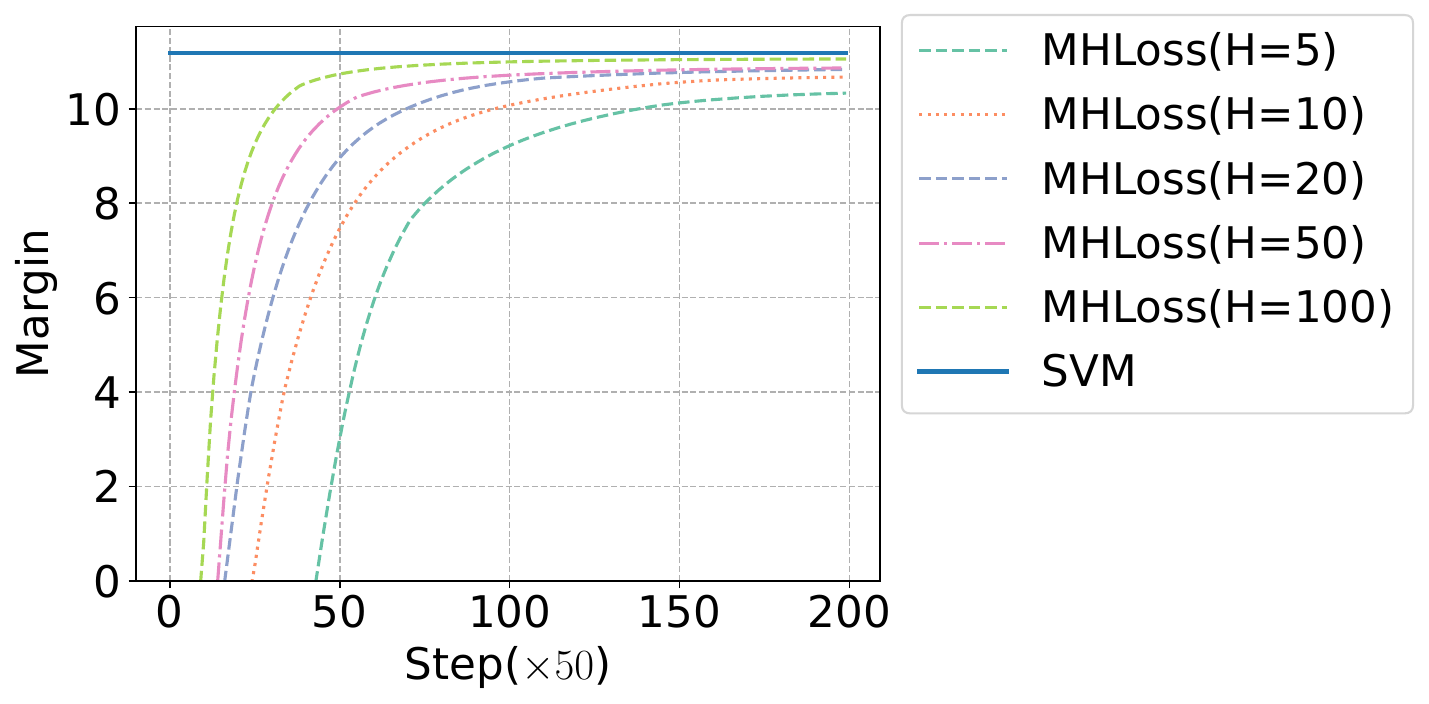}
    \caption{Influence of head number $H$.}
    \label{fig:head-coverage}
    \end{minipage}
    }
\end{figure*}

\subsubsection{Main Experiment}\label{sec:expmain} The main experiment results are shown in \Cref{tab:main}. It can be concluded from the results that using MHLoss helps stabilize the fine-tuning procedure since the output is more stable than that of the vanilla FineTune model. This observation verifies our theoretical analysis in \Cref{sec:mhloss}, where we prove that using more heads helps reduce the stability upper bound. We will conduct more experiments in \Cref{sec:head-std} to show how the head number $H$ affects the stability. Next, MMR strengthens the stability in many tasks, which verifies the prediction of \Cref{thm:mainthm} and our analysis in \Cref{sec:mmr}. It shows that increasing the margin of the encoded features can help improve stability. We will conduct more intuitive experiments to further verify this observation. Our proposed SURT model also outperforms the FineTune model by a notable margin in the standard deviation. This observation further shows the correctness of \Cref{thm:propgd}'s prediction discussed in \Cref{sec:surt}, where we prove that reducing the parameters' distance helps diminish the stability upper bound. Lastly, the LNSR model's results show that reducing the Lipschitz constant also helps improve stability. These results are consistent with the experiment results in \citet{hua2021noise} and also verify the correctness of our theoretical analysis in \Cref{thm:propgd} and \Cref{thm:thmlnsr}.

\begin{table}[t]
    \centering
    \scriptsize
    \setlength{\tabcolsep}{1pt} 
    \begin{tabular}{l|lllll|l}
    \toprule
    {} &                             \multicolumn{1}{c}{CoLA} &                             \multicolumn{1}{c}{MRPC} &                     \multicolumn{1}{c}{RTE} &                                  \multicolumn{1}{c}{CB} &                           \multicolumn{1}{c|}{WSC} &                           \multicolumn{1}{c}{AVG} \\
    \hline
    FineTune &                        \textbf{40.03}±1.69 &               \textbf{77.77}±0.73 &            \underline{56.72}±1.16 &                              75.52±6.31 &                              58.75±3.86 &                                    61.76±2.75 \\
MixOut   &                        \CCR{8.0}39.64±2.09 &              \CCG{11.6}77.09±0.65 &               \CCR{0.6}56.28±1.19 &        \CCG{24.6}76.43±\underline{5.58} &                    \CCG{10.0}58.75±3.86 &                          \CCG{11.6}61.64±2.67 \\
LNSR     &                        \CCR{4.6}39.43±1.92 &     \CCG{17.6}77.48±\textbf{0.35} &     \CCG{14.8}\textbf{57.44}±0.92 &                    \CCG{15.2}76.96±6.05 &                    \CCG{18.6}58.85±3.43 &                          \CCG{14.4}62.03±2.53 \\
MHLoss   &  \CCG{30.4}\underline{39.91}±\textbf{0.67} &               \CCR{7.2}77.42±1.09 &  \CCG{15.2}56.41±\underline{0.90} &  \CCG{26.2}\textbf{78.75}±\textbf{5.50} &  \CCG{64.4}\textbf{63.46}±\textbf{1.14} &        \CCG{27.8}\textbf{63.19}±\textbf{1.86} \\
MMR      &                        \CCR{7.8}39.11±2.08 &  \CCG{13.4}\underline{77.63}±0.56 &     \CCG{18.2}56.68±\textbf{0.75} &                    \CCG{13.6}78.57±6.13 &        \CCG{49.2}61.63±\underline{1.90} &  \CCG{19.4}\underline{62.73}±\underline{2.28} \\
SURT     &           \CCG{21.4}39.25±\underline{1.12} &  \CCG{13.6}77.44±\underline{0.55} &              \CCR{11.6}55.11±1.74 &           \CCG{22.8}\textbf{78.75}±5.67 &        \CCG{36.2}\underline{62.21}±2.55 &                          \CCG{18.4}62.55±2.33 \\
    \bottomrule
    \end{tabular}
    \caption{Head tuning stability analysis. 
    } 
    \label{tab:headtune}

\end{table}

\subsubsection{Impact of Head Number}\label{sec:head-std}  \Cref{thm:thmmhl} theoretically predicts that using MHLoss helps improve the fine-tuning stability, which has already been experimentally verified in \Cref{sec:expmain}. To further analyze how the head number $H$ affects the stability, we report standard deviations on several GLUE tasks with respect to different head numbers $H$ ranging in $\{1, 5, 50, 100, 200\}$. The results are shown in \Cref{fig:head-std}. It can be concluded from the results that as the head number $H$ increases, the standard deviation decreases. It shows the effectiveness of our proposed MHLoss and also empirically verifies the correctness of our theoretical prediction.

\subsubsection{Impact of Training Epoch}\label{sec:epoch-std} \Cref{thm:mainthm} indicates that training more epochs stabilizes fine-tuning, which has also been empirically verified by \citet{mosbach2020stability}. We further experimentally verify these results by fine-tuning the model with different epochs on several GLUE tasks and the results are shown in \Cref{fig:epoch-std}. We can conclude from the results that as the training epochs increase, the stability of nearly all tasks correspondingly improves. This observation further corroborates our theoretical conclusions.

\subsubsection{Impact of Learning Rate}\label{sec:lr-std} Different from other factors that directly reduce the stability upper bound, using a smaller learning rate is a necessary condition for \Cref{thm:mainthm} to hold. We conduct new experiments to show how the stability changes as different learning rates are applied. The results are shown in \Cref{fig:lr-std}. It can be concluded from the experiments that as the learning rate $\eta$ decreases, the model achieves better stability. This observation further justifies the correctness of our theoretical prediction and also validates the results observed in \citet{mosbach2020stability}.

\begin{table}[t]
    \centering
    \scriptsize
    \setlength{\tabcolsep}{1pt} 
    \resizebox{0.49\textwidth}{!}{
    \begin{tabular}{l|lllll|l}
    \toprule
    {} &                             \multicolumn{1}{c}{CoLA} &                             \multicolumn{1}{c}{MRPC} &                     \multicolumn{1}{c}{RTE} &                                  \multicolumn{1}{c}{CB} &                           \multicolumn{1}{c|}{WSC} &                           \multicolumn{1}{c}{AVG} \\
    \hline
    FineTune &                                 60.17±1.84 &                                    88.61±1.01 &            \underline{75.02}±1.56 &                              83.04±3.93 &                                 57.88±5.42 &                              72.94±2.75 \\
    MixOut   &                       \CCG{16.6}60.24±1.51 &              \CCG{13.4}88.59±\underline{0.84} &               \CCR{4.6}74.91±1.79 &        \CCG{30.0}83.75±\underline{2.93} &            \CCR{0.2}\underline{58.75}±5.43 &        \CCG{15.0}\underline{73.25}±2.50 \\
    LNSR     &           \CCG{21.6}60.11±\underline{1.26} &  \CCG{13.4}\underline{88.88}±\underline{0.84} &              \CCR{15.0}72.27±2.31 &                    \CCG{24.6}84.46±3.20 &                       \CCG{22.2}55.29±4.81 &                    \CCG{15.4}72.20±2.48 \\
    MHLoss   &  \CCG{22.8}\underline{60.35}±\textbf{1.20} &                          \CCG{13.2}88.41±0.85 &     \CCG{16.8}74.77±\textbf{1.22} &                    \CCG{28.4}85.00±3.01 &  \CCG{32.4}\textbf{61.92}±\underline{4.30} &  \CCG{22.6}\textbf{74.09}±\textbf{2.12} \\
    MMR      &              \CCG{19.0}\textbf{60.84}±1.39 &                 \CCG{12.4}\textbf{89.66}±0.89 &  \CCG{16.2}71.05±\underline{1.25} &        \CCG{20.0}\underline{85.54}±3.43 &              \CCG{33.2}56.73±\textbf{4.26} &        \CCG{20.2}72.76±\underline{2.24} \\
    SURT     &                       \CCG{13.2}60.04±1.68 &                 \CCG{16.2}88.67±\textbf{0.70} &     \CCG{15.4}\textbf{75.16}±1.29 &  \CCG{40.2}\textbf{86.79}±\textbf{2.42} &                        \CCR{1.0}54.13±5.47 &                    \CCG{18.8}72.96±2.31 \\
    \bottomrule
    \end{tabular}
    }
    \caption{Results for data perturbation stability experiments.
    }  
    \label{tab:datastablity}
\end{table}

\subsubsection{Impact of Sample Count}\label{sec:datasize-std} It has been indicated in both Theorems \ref{thm:propgd} and \ref{thm:mainthm} that using more training samples helps stabilize the fine-tuning procedure. We conduct a new experiment to verify this prediction by sampling the training set with ratios ranging in $\{40\%, 50\%, 60\%, 70\%, 80\%, 90\%, 100\%\}$. Then, we fine-tune the model with the sampled data and the results are shown in \Cref{fig:datasize-std}. It can be concluded that as we get more and more training samples, the models become more and more stable, which also corroborates our theory.

\subsubsection{Head Tuning Stability Analysis}\label{sec:headtune} To show that our proposed methods are also applicable to the head tuning settings, we run the same experiments again in the head tuning manner and the results are shown in \Cref{tab:headtune}. In this setting, the pre-trained encoders are fixed and only the head layer's parameters are fine-tuned. It can be observed from the results that in the head tuning setting, our proposed methods also improve stability. This result also validates our theoretical prediction.

\subsubsection{Data Perturbation Stability}
\label{sec:datastablity} In the main experiment (\Cref{sec:mainexp}), we perturb the training data by switching the random seeds. It is still unknown whether it will remain stable if we impose some perturbation on the training data. To verify this kind of stability, we conduct a new experiment by training the model on several datasets with 10\% of their training samples randomly removed. The results are shown in \Cref{tab:datastablity}. It can be concluded from the results that our proposed new methods including MHLoss, MMR, and SURT can help stabilize the training procedure with perturbation on the training data. This observation further extends our methods to scenarios with perturbation on input data. Furthermore, existing methods MixOut and LNSR stabilize fine-tuning compared with the FineTune model, which also supports our theoretical prediction.

\subsubsection{Large Pre-Trained Model Stability}\label{sec:robertalarge} In the above discussion, we conduct extensive experiments on RoBERTa base model. To further explore whether our theoretical results are applicable for fine-tuning a large backbone model, we run several tasks with the RoBERTa-large model \cite{liu2019roberta}. The results are shown in \Cref{tab:robertalarge}. It can be concluded from the results that the scores for most of the experiments improve as we use a much larger backbone model. Besides, all stabilizing methods including MHLoss, MMR, SURT, and previously proposed LNSR reduce the variance of the FineTune model. This observation shows that our proposed new methods are also applicable to a large pre-trained model. This experiment also extends the application scenarios of our proposed methods.

\begin{table}[t]
    \centering
    \scriptsize
    \setlength{\tabcolsep}{1pt} 
    \begin{tabular}{l|lllll|l}
    \toprule
    {} &                             \multicolumn{1}{c}{CoLA} &                             \multicolumn{1}{c}{MRPC} &                     \multicolumn{1}{c}{RTE} &                                  \multicolumn{1}{c}{CB} &                           \multicolumn{1}{c|}{WSC} &                           \multicolumn{1}{c}{AVG} \\
    \hline
    FineTune     &                        64.63±2.19 &                                    90.41±0.84 &                        83.42±2.05 &                        84.95±5.13 &                                    67.60±3.25 &                              78.20±2.69 \\
    MixOut       &  \CCG{16.4}\underline{65.19}±1.87 &                           \CCR{0.8}90.48±0.88 &               \CCR{5.6}84.04±2.33 &  \CCG{23.0}85.54±\underline{4.48} &                 \CCG{26.4}66.35±\textbf{2.43} &                    \CCG{15.8}78.32±2.40 \\
    LNSR         &  \CCG{32.4}64.80±\underline{1.07} &                 \CCG{12.6}\textbf{90.57}±0.71 &  \CCG{29.6}82.95±\underline{1.07} &              \CCG{16.8}84.13±4.79 &                          \CCR{10.4}66.06±3.77 &                    \CCG{18.2}77.70±2.28 \\
    MHLoss       &     \CCG{29.0}\textbf{65.48}±1.24 &                 \CCG{17.8}90.44±\textbf{0.45} &     \CCG{31.6}84.26±\textbf{0.97} &  \CCG{15.2}\underline{86.61}±4.87 &                          \CCG{17.0}64.33±2.90 &        \CCG{22.0}78.22±\underline{2.09} \\
    MMR          &              \CCG{18.8}63.45±1.75 &  \CCG{15.6}\underline{90.54}±\underline{0.56} &     \CCG{21.2}\textbf{84.51}±1.49 &     \CCG{31.4}85.36±\textbf{4.06} &  \CCG{25.2}\underline{68.17}±\underline{2.49} &  \CCG{22.4}\textbf{78.41}±\textbf{2.07} \\
    SURT &     \CCG{33.0}60.75±\textbf{1.04} &                          \CCG{12.8}90.09±0.70 &  \CCG{19.6}\underline{84.30}±1.57 &      \CCR{1.0}\textbf{87.76}±5.18 &                 \CCR{17.0}\textbf{68.94}±4.10 &        \CCG{13.4}\underline{78.37}±2.52 \\
    \bottomrule
    \end{tabular}
    \caption{Results for fine-tuning a RoBERTa large model. 
    }    
    \label{tab:robertalarge}
\end{table}


\subsection{Experiment Results for Synthetic Classification}\label{sec:simplecls}

To provide a more complete picture of how internal factors (such as margins and distances) affect stability, we manually create a series of synthetic binary classification tasks. These factors can be manually controlled in these datasets. The training samples for each task are randomly generated with regard to a particular factor level and are classified with a linear classifier. The synthetic datasets have many advantages. First, the value of each factor is controllable and we can easily show how the stability is influenced by different factor values. Besides, we can generate large numbers of different datasets to achieve better statistical significance.
If not particularly specified, we randomly generate 500 training sets and each training set contains 2,000 data points. As discussed in \Cref{sec:headtuneanalysis}, the normalized leave-one-out stability is more suitable for analyzing linear classifiers. We write $\|\hat{\Delta}\|=\left\|\frac{\hat{w}_S}{\|\hat{w}_S\|} - \frac{\hat{w}_{S^i}}{\|\hat{w}_{S^i}\|}\right\|$, where $\|\hat{\Delta}\|$ measures the normalized leave-one-out model stability as defined in \Cref{def:nstability}. We train linear regression models on each training set with the gradient descent method and report metrics accordingly.

\subsubsection{Impact of Head Number}\label{sec:head-coverage} We train the model with a linear regression model with the MHLoss. The head number $H$ ranges in $\{5, 10, 20, 50, 100\}$. The results are shown in \Cref{fig:head-coverage}, where the SVM model is a max-margin classifier. It can be concluded that using more heads leads to faster convergence than using fewer heads. This observation justifies the discussion of \Cref{thm:thmmhl} that MHLoss accelerates the convergence of the first term in \Cref{eqn:mhloss} to improve stability. Besides, all the linear regression models converge to a max-margin classifier. This observation gives a quick verification of the theory in \citet{soudry2018implicit}, which forms the basis of our theoretical analysis.

\subsubsection{The Impact of Sample Count}\label{sec:N-pertube} To further show how increasing the sample count contributes to stabilizing the fine-tuning procedure, we train models on synthetic datasets with different sample sizes and the results are shown in \Cref{fig:N-pertube}. It can be concluded that $\|\hat{\Delta}\|$ decreases as the sample size increases, which indicates a more stable training procedure. This observation further verifies \Cref{thm:propgd} and \ref{thm:mainthm}'s prediction.

\subsubsection{Impact of Sample Margin}\label{sec:margin-pertube} \Cref{thm:mainthm} indicates that increasing the margin between features can improve stability. To give a more intuitive view of how the margin influences the stability of the linear head, we conduct a new experiment to illustrate the relationship between the margin and stability. We manually adjust the data distribution to control the distance between the center points of the two generated sample classes and thus control the margin. Then, we calculate the margin with a simple SVM model and plot the relationship between the margin and the corresponding stability metric. The results are shown in \Cref{fig:margin-pertube}. It can be concluded that as the margin increases, the stability improves, which also justifies our theoretical prediction.

\begin{figure*}[t]
    \makebox[\linewidth][c]{
    \centering
    \begin{minipage}[t]{0.265\textwidth}
    \centering
    \includegraphics[width=0.99\columnwidth]{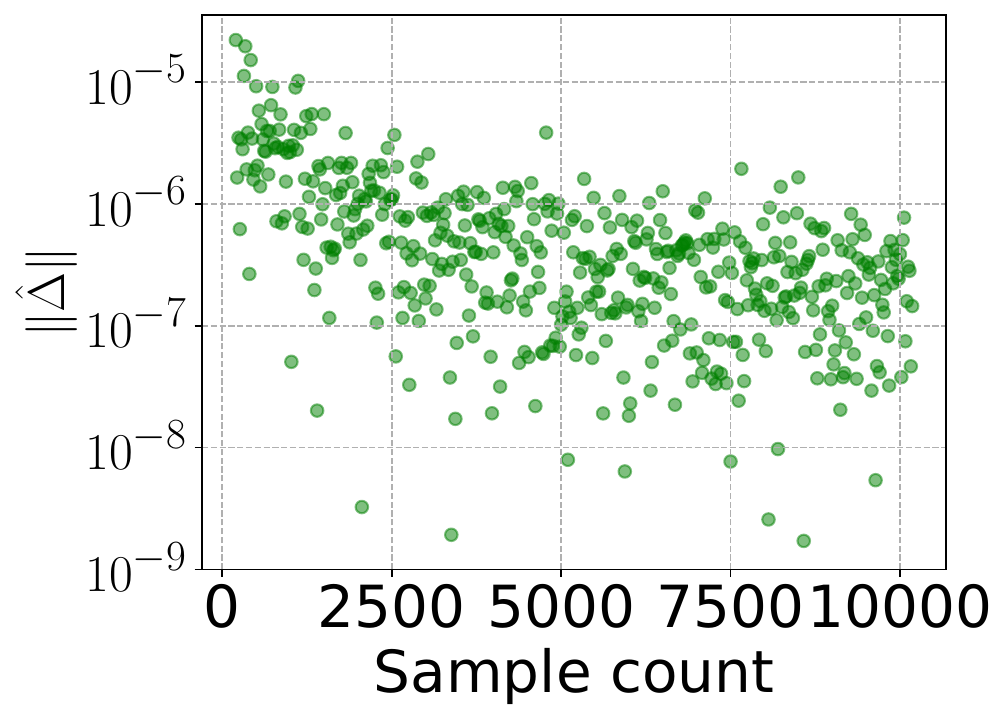}
    \begin{minipage}[t]{0.9\columnwidth}
    \caption{Influence of sample count}
    \label{fig:N-pertube}
    \end{minipage}
    \end{minipage}
    \begin{minipage}[t]{0.25\textwidth}
        \centering
        \includegraphics[width=0.99\columnwidth]{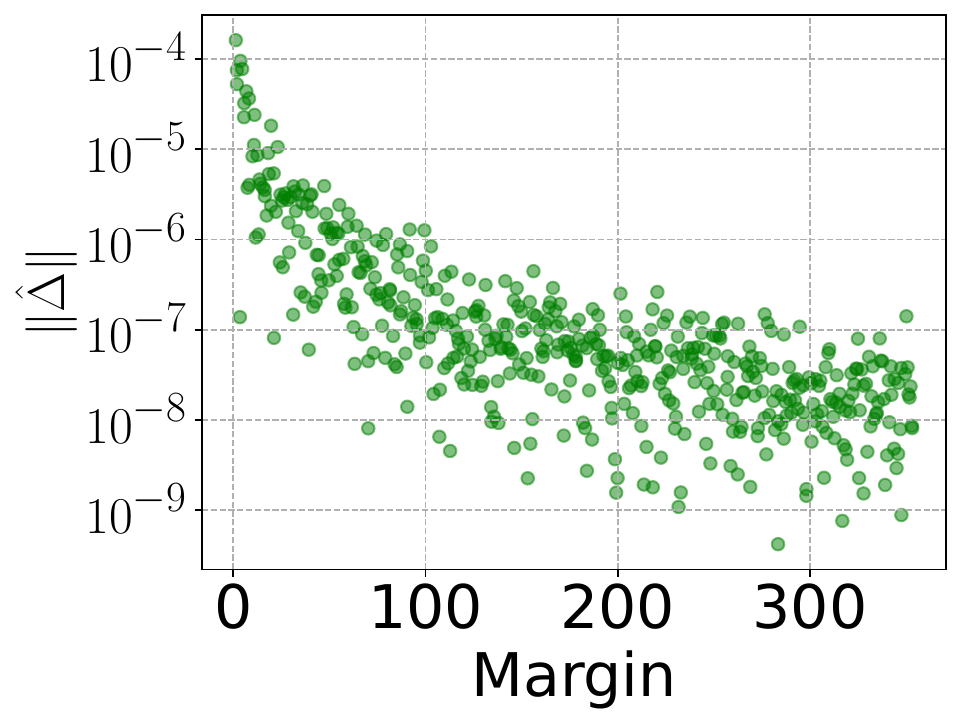}
    \begin{minipage}[t]{0.9\columnwidth}
    \caption{Influence of sample margin.}
    \label{fig:margin-pertube}
    \end{minipage}
    \end{minipage}
    \begin{minipage}[t]{0.243\textwidth}
        \centering
        \includegraphics[width=0.99\columnwidth]{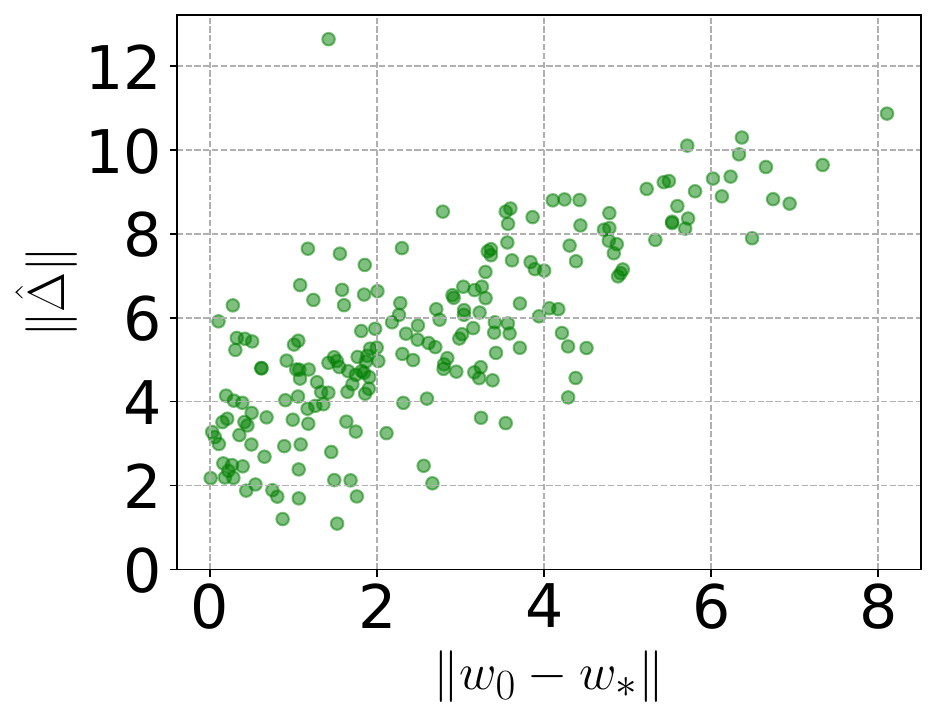}
    \begin{minipage}[t]{0.9\columnwidth}
    \caption{Influence of parameters distance.}
    \label{fig:dist-stable}
    \end{minipage}
    \end{minipage}
    \begin{minipage}[t]{0.251\textwidth}
        \centering
        \includegraphics[width=0.99\columnwidth]{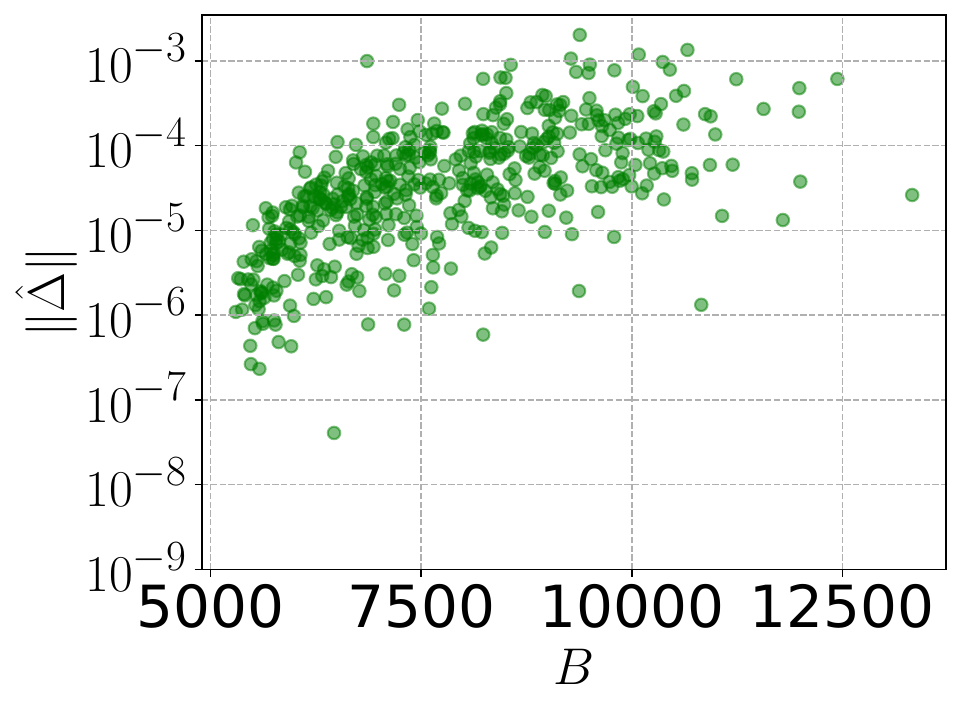}
    \begin{minipage}[t]{0.9\columnwidth}
    \caption{Influence of sample bound.}
    \label{fig:Bs-pertube}
    \end{minipage}
    \end{minipage}
    }
\end{figure*}

\subsubsection{Impact of Parameters Distance} \label{sec:dist-stable} In \Cref{thm:propgd}, we show that if the distance ($\|w_0-w_*\|$) between the initialization parameter $w_0$ and the optimal parameter $w_*$ is reduced, the stability can be improved. We conduct a new experiment to further verify this prediction. As shown in \Cref{fig:dist-stable}, we start from a random initialization point $w_0$ and use gradient descent to get the optimal point $w_*$. We also calculate $\|\hat{\Delta}\|$ to measure the stability. It can be observed that if the distance $\|w_0-w_*\|$ is small, the training procedure becomes more stable. This observation further shows the effectiveness of our proposed SURT model, as well as verifies the prediction of our theory.

\subsubsection{The Impact of Sample Bound} In \Cref{thm:mainthm}, we find that the sample bound $B$ also influences the stability. To verify this prediction, we train the classifiers on a bunch of generated datasets. We draw a plot of the stability score with respect to different sample bound $B$. The results are shown in \Cref{fig:Bs-pertube}. It can be observed that as $B$ decreases, the model becomes more stable, which further verifies the correctness of our theory.

\section{Related Works}

Many works have been proposed to stabilize the fine-tuning procedure. 
\citet{cattan2021usability,mosbach2020stability} show that fine-tuning the model with more iterations can make it more stable, while
\citet{arora2018stronger,sanyal2019stable,hua2021noise,aghajanyan2020better} propose to use a noise stability regularization to stabilize fine-tuning.
On the other hand, \citet{zhang2020revisiting,cattan2021usability,mosbach2020stability} find that using a small dataset leads to instability, while \citet{cattan2022usability} show that data augmentation helps improve the stability.
Moreover, \citet{han2021robust} propose to train the adapter separately to improve the stability, while \citet{yang2022improving} propose a componentwise gradient norm clipping method to improve it. Besides, \citet{he2021effectiveness,lee2019mixout,houlsby2019parameter,zaken2021bitfit,sung2021training,liu2021p} find that tuning part of the pre-trained parameters also helps stabilize the model. However, to the best of our knowledge, there are no theoretical works that analyze the effectiveness of these fine-tuning methods. 

Stability of training a model has been studied for decades. \citet{bousquet2002stability,shalev2010learnability,shalev2014understanding,charles2018stability} propose the standard stability definition for general machine learning models making it analyzable. \citet{hardt2016train,kuzborskij2018data} propose to analyze the stability of stochastic gradient methods while \citet{lei2020fine,schliserman2022stability} propose the Leave-one-out Model Stability which directly checks the distance between trained parameters. We extend the stability analysis to the fine-tuning regime and design several effective methods based on our new theory.

\section{Conclusion}
In this paper, we propose a novel theoretical analysis of the stability of fine-tuning a pre-trained model. We first define theoretical stability bounds in two commonly used settings, namely, the full fine-tuning and the head tuning. Then, we give a theoretical analysis that provides the basis for four existing and widely used methods proposed by previous works. In addition to being able to explain most of the observed empirical discoveries, our theory can help in the design of efficient and provable methods. Based on our theory, we propose Max Margin Regularizer (MMR), Multi-Head Loss (MHLoss), and Self Unsupervised Re-Training (SURT) methods to stabilize fine-tuning. We conduct extensive experiments on 11 widely used real-world datasets together with extensive experiments on a bunch of synthetic classification datasets. 
The experiment results show the effectiveness of our proposed methods and hence validate our theory as well.

\bibliographystyle{acl_natbib}
\bibliography{reference}


\clearpage
\appendix
\begin{center}
{\LARGE \textbf{Appendix. Supplementary Material}}
\end{center}
\renewcommand{\thesection}{A.\arabic{section}}
\setcounter{theorem}{0}
\setcounter{lemma}{0}
\setcounter{corollary}{0}
\setcounter{proposition}{0}

\section{Proof of \Cref{thm:propgd}}\label{sec:A-propgd}

To prove \Cref{thm:propgd}, we first recall three technical lemmas that will be used in the proof. We say function $f$ is $\mu$-strongly convex if $f(y) \geq f(x)+\langle\nabla f(x), y-x\rangle+\frac{\mu}{2}\|y-x\|_2^2, \quad \forall x, y \in \mathbb{R}^d$. We denote $w_t$ as the model parameter $w$ at time $t$ trained on the data $S$. Similarly, $w^i_t$ is the model parameter $w$ at time $t$ trained on the data $S^i$. We first recall the gradient descent convergence bound for the $\mu$-strongly convex function given by \Cref{lemma:scbound}. Then, we give a standard lemma (\Cref{lemma:nonexpansion}) widely used in smooth convex optimization, as well as a lemma for a self-bounded function in \Cref{lemma:selfbound}. Finally, we give the proof of \Cref{thm:propgd} which focuses on the Taylor expansion of the full fine-tuning.

\begin{lemma}[\citet{nesterov2018lectures}]
Let $f(w)$ be a $\beta$-smooth and $\mu$-strongly convex function. For a given point $w_0 \in \mathbb{R}^d$ and $\frac{1}{\beta} \geq \eta>0$, the gradient descent iterates
$$
w_{t+1}=w_t-\eta \nabla f\left(w_t\right)
$$
converge according to
$$
\left\|w_{t}-w_*\right\|_2^2 \leq(1-\eta \mu)^{t}\left\|w_0-w_*\right\|_2^2 .
$$
\label{lemma:scbound}
\end{lemma}

We also recall the following standard lemma for smooth convex optimization. It is a natural consequence of co-coercivity and we give a simplified proof here. More details can be found in \citet{hardt2016train,schliserman2022stability}. 

\begin{lemma}
    If $f: \mathbb{R}^d \rightarrow \mathbb{R}$ is convex and $\beta$-smooth and $0<\eta \leq 2 / \beta$, then for every $u, v \in \mathbb{R}^d$,
$$
\|(u-\eta \nabla f(u))-(v-\eta \nabla f(v))\| \leq\|u-v\| .
$$
\label{lemma:nonexpansion}
\end{lemma}

\begin{proof}
    We first recall the co-coercivity of a $\beta$-smooth function. For a  $\beta$-smooth function $f$, $\forall x, y$,
    \[ \langle \nabla f(x) - \nabla f(y), x - y \rangle \geq \frac{1}{\beta} \|\nabla f(x) - \nabla f(y)\|^2. \tag*{(co-coercivity)}\]

    Then, we have 
    $$\begin{aligned}
        \|(u-\eta \nabla f(u))-(v-\eta \nabla f(v))\|^2 &= \|(u-v)-\eta (\nabla f(u)- \nabla f(v))\|^2\\
        &=\|u-v\|^2+\eta^2 \|\nabla f(u)-\nabla f(v)\|^2-2\eta \nabla \langle u-v, \nabla f(u)- \nabla f(v)\rangle\\
        &\le \|u-v\|^2+\eta^2 \|\nabla f(u)-\nabla f(v)\|^2-\frac{2\eta}{\beta}\|\nabla f(u)-\nabla f(v)\|^2&&\text{(by co-coercivity)}\\
        &= \|u-v\|^2+(\eta^2-\frac{2\eta}{\beta}) \|\nabla f(u)-\nabla f(v)\|^2\\
        &\le \|u-v\|^2 + 0 \|\nabla f(u)-\nabla f(v)\|^2\\
        &= \|u-v\|^2
    \end{aligned}
    $$
    Therefore, we have $$
    \|(u-\eta \nabla f(u))-(v-\eta \nabla f(v))\| \leq\|u-v\| .
    $$
The proof is complete.
\end{proof}

Then, following \citet{shalev2014understanding}, a non-negative $\beta$-smooth function is self-bounded. We give a full proof here.

\begin{lemma}
    If $f: \mathbb{R}^d \rightarrow \mathbb{R}$ is a non-negative twice differentiable and $\beta$-smooth function. Then 
    $$\|\nabla f(w)\|^2 \leq 2 \beta f(w).$$
    \label{lemma:selfbound}
\end{lemma}

\begin{proof}
    By the first order expansion of $f(w)$, for a small incremental vector $\delta$ we have that
    $$f(w+\delta)=f(w)+\nabla f(w)^{\T}\delta$$
    Therefore, take the integral from $v$ to $w$ with a new auxiliary variable $t\in [0, 1]$, we have
    $$
    \begin{aligned}
    f(w) & =f(v)+\int_{0}^1\langle\nabla f(v+t(w-v)),(w-v)\rangle d t . \\
    & =f(v)+\langle\nabla f(v), w-v\rangle+\int_{0}^1\langle\nabla f(v+t(w-v))-\nabla f(v),(w-v)\rangle d t . \\
    & \leq f(v)+\langle\nabla f(v), w-v\rangle+\int_{0}^1\|\nabla f(v+t(w-v))-\nabla f(v)\|\|w-v\| d t \\
    & \leq f(v)+\langle\nabla f(v), w-v\rangle+\beta \int_{0}^1 t\|w-v\| d t \\
    & \leq f(v)+\langle\nabla f(v), w-v\rangle+\frac{\beta}{2}\|w-v\|^2
    \end{aligned}
    $$
    Then, take $w=v-\frac{1}{\beta}\nabla f(v)$, we have
    $$
    \begin{aligned}
        f(v-\frac{1}{\beta}\nabla f(v))&\le f(v)+\langle\nabla f(v), v-\frac{1}{\beta}\nabla f(v)-v\rangle+\frac{\beta}{2}\|v-\frac{1}{\beta}\nabla f(v)-v\|^2\\
        0\le f(v-\frac{1}{\beta}\nabla f(v))&\le f(v)+\langle\nabla f(v), -\frac{1}{\beta}\nabla f(v)\rangle+\frac{\beta}{2}\|-\frac{1}{\beta}\nabla f(v)\|^2\\
        0&\le f(v)-\frac{1}{\beta} \|\nabla f(v)\|^2+\frac{1}{2\beta}\|\nabla f(v)\|^2\\
        \|\nabla f(v)\|^2&\le 2\beta f(v)
    \end{aligned}
    $$
The proof is complete.
\end{proof}

Finally, We prove \Cref{thm:propgd} based on the above lemmas.

\propgd*
\begin{proof}

Then, similar to \citet{schliserman2022stability}, as the gradient descent iterates as $$w_{t+1}=w_t-\eta \frac 1 n \sum_{i=1}^n \nabla f(w_t,z_i),$$ 

we have

\begin{equation}
\begin{aligned}
\left\|w_{t+1}-w_{t+1}^i\right\| &=\left\|w_t-\frac{\eta}{n} \sum_{j=1}^n \nabla f\left(w_t, z_j\right)-w_t^i+\frac{\eta}{n} \sum_{j \neq i} \nabla f\left(w_t^i, z_j\right)\right\| \\
&=\left\| \frac{1}{n} \sum_{j \neq i}\left(w_t-\eta \nabla f\left(w_t, z_j\right)-w_t^i+\eta \nabla f\left(w_t^i, z_j\right)\right)+\frac{1}{n} w_t-\frac{1}{n} w_t^i-\frac{\eta}{n} \nabla f\left(w_t, z_i\right)\right \| \\
& \leq \frac{1}{n} \sum_{j \neq i}\left\|w_t-\eta \nabla f\left(w_t, z_j\right)-w_t^i+\eta \nabla f\left(w_t^i, z_j\right)\right\|+\frac{1}{n}\left\|w_t-\eta \nabla f\left(w_t, z_i\right)-w_t^i\right\|\\
&\leq \frac{1}{n} \sum_{j \neq i}^n\left\|w_t-w_t^i\right\|+\frac{1}{n}\left\|w_t-w_t^i\right\|+\frac{\eta}{n}\left\|\nabla f\left(w_t, z_i\right)\right\| \ \ \ \ \ \ \ &&\text{(\Cref{lemma:nonexpansion})}\\
&\leq\left\|w_t-w_t^i\right\|+\frac{\eta \sqrt{2\beta}}{n} \sqrt{f\left(w_t, z_i\right)} \ \ \ \ \ \ \ &&\text{(\Cref{lemma:selfbound})}
\end{aligned}
\label{eqn:etadiv}
\end{equation}

For a given step $T$, we have
$$\begin{aligned}
\left\|w_{T}-w_{T}^i\right\|&\le\left\|w_{T-1}-w_{T-1}^i\right\|+\frac{\eta \sqrt{2\beta}}{n} \sqrt{f\left(w_{T-1}, z_i\right)} \\
&\le\left\|w_{T-2}-w_{T-2}^i\right\|+\frac{\eta \sqrt{2\beta}}{n} \sqrt{f\left(w_{T-1}, z_i\right)} +\frac{\eta \sqrt{2\beta}}{n} \sqrt{f\left(w_{T-2}, z_i\right)} \\
& \cdots \\
&\le \left\|w_0-w_0^i\right\|+\frac{\eta \sqrt{2\beta}}{n} \sum_{t=1}^{T-1} \sqrt{f\left(w_{t}, z_i\right)}\\
&\le \frac{\eta \sqrt{2\beta}}{n} \sum_{t=1}^{T} \sqrt{f\left(w_{t}, z_i\right)}\\
&= \frac{ \sqrt{2/\beta}}{n} \sum_{t=1}^{T} \sqrt{f\left(w_{t}, z_i\right)}\\
\end{aligned}$$

As we analyze the second order Taylor approximation of $f(w_*,z)$ and $0 \prec \mu I\preceq \nabla^2f(w_*,z)$, the approximation  is $\mu$-strongly convex and the Hessian of the approximation is positive definite for all $w$'s. Then, from \Cref{lemma:scbound}, we have
$$\begin{aligned}
f(w_{t}, z_i)&= f(w_{t}, z_i)- 0\\
&= f(w_{t}, z_i)- f(w_{*}, z_i)\\
&\le L\|w_{t}-w_*\|\\
&\le L(1-\eta \mu)^{t/2}\left\|w_0-w_*\right\|
\end{aligned}$$

As assumed, when we take $\eta=\frac 1 \beta$, the summation will have a better convergence property and will not diverge to infinity. Thus $f(w_{t}, z_i)\le L(1- \frac \mu \beta)^{t/2}\left\|w_0-w_*\right\|$. Therefore,

$$\begin{aligned}
\|w_{T}-w_{T}^i\|&\le \frac{\sqrt{2/\beta}}{n} \sum_{t=1}^{T} \sqrt{f\left(w_{t}, z_i\right)}\\
&\le \frac{ \sqrt{2L/\beta }}{n} \sqrt{\|w_0-w_*\|}\sum_{t=1}^{T} (1-\frac \mu \beta)^{t/4}\\
&\le\frac{ \sqrt{2L/\beta}}{n} \sqrt{\|w_0-w_*\|}\sum_{t=1}^{\infty} (1-\frac \mu \beta)^{t/4}\\
&=\frac{ \sqrt{2L\|w_0-w_*\|/\beta}}{n} \frac{\sqrt[4]{1-\frac \mu \beta}}{1-\sqrt[4]{1-\frac \mu \beta}}\\
&=\frac{ \sqrt{2L\|w_0-w_*\|/\beta}}{n(1/\sqrt[4]{1-\frac \mu \beta}-1)} 
\end{aligned}$$
Similar to \citet{shalev2014understanding}, since this holds for any $S, T, i$ we immediately obtain:
$$\begin{aligned}
\E_{S}[\|\A(S^{i})-\A(S)\|]&\le \frac{ \sqrt{2L\|w_0-w_*\|/\beta}}{n(1/\sqrt[4]{1-\frac \mu \beta}-1)}  . 
\end{aligned}$$
The proof is complete.
\end{proof}

\section{Proof of \Cref{thm:mainthm}}\label{sec:A-mainthm}

To prove \Cref{thm:mainthm}, we first introduce several technical lemmas. To analyze the SVM classifier, we use $w=[v^\T,b]^\T$ to denote all the parameters of the linear classifier. $v\in \mathbb{R}^{d_x}$ is the weight vector while $b$ is the intercept. Then, the classifier can be written as $y=\operatorname{sign}(v^\T x-b)$. We denote $\hat{w}_{S}=[\hv_S^\T,\hb_S]^\T$ as the SVM solution trained on dataset $S^i$ and denote $\hat{w}_{S^i}=[\hv_{S^i}^\T,\hb_{S^i}]^\T$ as the SVM solution trained on dataset $S^i$. First, in \Cref{lemma:mlip}, we study the Lipschitz constant for function, $f(x)=x/\|x\|$. Generally, this function is not Lipschitz continuous. However, if $\|x\|$ is bounded below, we can get its Lipschitz constant. In \Cref{lemma:cvx}, we study some properties of the function $f(v,b)=\max(0, 1 - y(v^\T x - b))$. We found it is a convex function and calculate the bound for $|f(v_1,b_1)-f(v_2,b_2)|$. In \Cref{lemma:bbound}, we study the bound for the intercept $\hb$. Then, in \Cref{lemma:maxw}, we compare the optimal solutions $\hat{v}_{S}$ and $\hat{v}_{S^i}$ which is the maximal margin classifier weight for dataset $S$ and $S^i$ respectively. They have a tight relationship with the margins. Finally, equipped with these lemmas, we give the proof of \Cref{thm:mainthm}.

\begin{lemma}
Given function $f(x)=\frac{x}{\|x\|}$ defined on domain $\{x|\|x\|\ge r, r>0\}$. The Lipschitz constant is $\nu=\frac{2}{r}$. Namely,
$$\|f(x)-f(y)\|\le \nu \|x-y\|.$$
\label{lemma:mlip}
\end{lemma}

\begin{proof}
$$\begin{aligned}
\|f(x)-f(y)\|&=
\left\|\frac{x}{\|x\|}-\frac{y}{\|y\|}\right\|
=\frac{1}{\|x\|\|y\|}\big\|\|y\|x-\|x\|y\big\| \\ 
&= \frac{1}{\|x\|\|y\|}\big\|\|y\|(x-y)-(\|x\|-\|y\|)y\big\| \\
&\le \frac{1}{\|x\|\|y\|}\big\|\|y\|(x-y)\big\|+
\frac{1}{\|x\|\|y\|}\big\|(\|x\|-\|y\|)y\big\| \\  
&\le\frac{1}{\|x\|}\big\|x-y\big\| + \frac{1}{\|x\|}\big|\|x\|-\|y\|\big|\\
&=\frac{1}{\|x\|}\big\|x-y\big\| + \frac{1}{\|x\|}\sqrt{\big(\|x\|-\|y\|\big)^2}\\
&=\frac{1}{\|x\|}\big\|x-y\big\| + \frac{1}{\|x\|}\sqrt{\|x\|^2-2\|x\|\|y\|+\|y\|^2}\\
&\le\frac{1}{\|x\|}\big\|x-y\big\| + \frac{1}{\|x\|}\sqrt{\|x\|^2-2x^{\T}y+\|y\|^2}\\
&=\frac{1}{\|x\|}\big\|x-y\big\| + \frac{1}{\|x\|}\sqrt{\|x-y\|^2}\\
&=\frac{1}{\|x\|}\big\|x-y\big\| + \frac{1}{\|x\|}\big\|x-y\big\|\\
&\le \frac{2}{r}\big\|x-y\big\|\\
&= \nu\big\|x-y\big\|.
\end{aligned}$$
\label{lemma:normlip}
The proof is complete.
\end{proof}

\begin{lemma}
Given a function $f(v,b)=\max(0, 1 - y(v^\T x - b))$, $\forall x\in \mathbb{R}^{d_x}, y\in \{-1,1\}$. We can conclude that $f(v,b)$ is convex and
$$|f(v_1,b_1)-f(v_2,b_2)|\le \max(1+\|v_1\|\|x\|+|b_1|,1+\|v_2\|\|x\|+|b_2|,\|v_1-v_2\|\|x\|+|b_1|+|b_2|).$$

\label{lemma:cvx}
\end{lemma}

\begin{proof}
    We first show that $f(v,b)$ is convex. It can easily observed that function $g_1(v,b)=0$ and function $g_2(v,b)=1-y(v^\T x-b)$ is also convex. Therefore, the pointwise maximum $f(v,b) = \max ( g_1 (v,b), g_2 (v,b) )=\max(0, 1-y(v^\T x-b))$ is convex \cite{boyd2004convex}.
    
    Then, we prove the inequality.

    If $1 - y(v_1^\T x - b_1)\le 0$ and $1 - y(v_2^\T x - b_2)\le 0$:
    $$|\max(0, 1 - y(v_1^\T x - b_1))-\max(0, 1 - y(v_2^\T x - b_2))|=0.$$
    
    If $1 - y(v_1^\T x - b_1)> 0$ and $1 - y(v_2^\T x - b_2)\le 0$:
    $$|\max(0, 1 - y(v_1^\T x - b_1))-\max(0, 1 - y(v_2^\T x - b_2))|=|1 - y(v_1^\T x - b_1)|\le 1+\|v_1\|\|x\|+|b_1|.$$
    
    If $1 - y(v_1^\T x - b_1)\le 0$ and $1 - y(v_2^\T x - b_2)> 0$:
    $$|\max(0, 1 - y(v_1^\T x - b_1))-\max(0, 1 - y(v_2^\T x - b_2))|=|1 - y(v_2^\T x - b_2)|\le 1+\|v_2\|\|x\|+|b_2|.$$
    
    If $1 - y(v_1^\T x - b_1)> 0$ and $1 - y(v_2^\T x - b_2)> 0$:
    $$|\max(0, 1 - y(v_1^\T x - b_1))-\max(0, 1 - y(v_2^\T x - b_2))|=| - y(v_1^\T x - b_1)+ y(v_2^\T x - b_2)|\le \|v_1-v_2\|\|x\|+|b_1|+|b_2|.$$

    Therefore 
    $$\begin{aligned}
        |f(v_1,b_1)-f(v_2,b_2)|&=|\max(0, 1 - y(v_1^\T x - b_1))-\max(0, 1 - y(v_2^\T x - b_2))|\\
        &\le \max(1+\|v_1\|\|x\|+|b_1|,1+\|v_2\|\|x\|+|b_2|,\|v_1-v_2\|\|x\|+|b_1|+|b_2|)
    \end{aligned}$$
    The proof is complete.
\end{proof}

\begin{lemma}
    Given a linearly separable dataset $S=\{(x_1,y_1),(x_2,y_2),\cdots,(x_n,y_n)\}$, $x_i\in \mathbb{R}^{d_x}, \|x_i\|\le B, y_i\in \{-1,1\}$, the intercept $\hb$ for the maximal margin classifier $y=\operatorname{sign}(\hv^\T x-\hb)$ satisfies
    $$\|\hb\|\le 1+\|\hv\|B$$
    \label{lemma:bbound}
\end{lemma}
\begin{proof}
    The maximal margin classifier $\hv,\hb$ is the solution of the following problem:
    $$\begin{aligned}
        &\min_{v,b} && \|v\|^2\\
        &\text{s.t.} && y_i(v^\T x_i - b) \geq 1 \quad \forall i \in \{1,\dots,n\}
    \end{aligned}$$
    As the dataset $S$ is linearly separable, the feasible domain is non-empty and $\hb$ satisfies all the constraints. $\forall y_i = 1$, such that 
    $$\begin{aligned}
        y_i(\hv^\T& x_i - \hb) \geq 1\\
        \hv^\T x_i& - \hb \geq 1\\
        \|\hb\| &\le \|\hv^\T x_i-1\| \\
        &\le 1+\|\hv\|\|x_i\| \\
        &\le 1+\|\hv\|B
    \end{aligned}$$
    The proof is complete.
\end{proof}

\begin{lemma}
Given a linearly separable dataset $S=\{(x_1,y_1),(x_2,y_2),\cdots,(x_n,y_n)\}$, $S^i=\{(x_1,y_1),\cdots,(x_{i-1},y_{i-1}),$ $(x_{i+1},y_{i+1}),\cdots,(x_n,y_n)\}$, $x_i\in \mathbb{R}^{d_x}, \|x_i\|\le B, y_i\in \{-1,1\}$. We denote $\hat{w}_{S}=[\hv_S^\T,\hb_S]^\T$ as the SVM solution trained on dataset $S^i$ and denote $\hat{w}_{S^i}=[\hv_{S^i}^\T,\hb_{S^i}]^\T$ as the SVM solution trained on dataset $S^i$. We have
$$\|\hat{v}_S\| \ge \|\hat{v}_{S^i}\|$$
\label{lemma:maxw}
\end{lemma}
\begin{proof}
By SVM's calculation, $\hat{v}_S$ can be calculated as
\begin{equation}
    \begin{aligned}
        \hv_S=&\min_{v,b}  \|v\|^2\\
        \text{s.t.}\ \ \ \ \ &  y_i(v^\T x_i - b) \geq 1 \quad \forall i \in \{1,\dots,n\}
    \end{aligned}
    \label{eqn:gammaS}
\end{equation}

On the other hand, $\hat{v}_{S^i}$ can be calculated as
\begin{equation}  
\begin{aligned}
    \hat{v}_{S^i}=&\min_{v,b}  \|v\|^2\\
    \text{s.t.}\ \ \ \ \ &  y_i(v^\T x_i - b) \geq 1 \quad \forall i \in \{1,\dots,i-1,i+1,\cdots,n\}
    \label{eqn:gammaSi}
\end{aligned}
\end{equation}
\Cref{eqn:gammaS} has one more constraint $y_i(v^\T x_i - b) \geq 1 $ than \Cref{eqn:gammaSi}. Therefore $\hv_S, \hb_S$ will always be a feasible solution to \Cref{eqn:gammaSi}. Therefore $\|\hat{v}_{S}\| \ge \|\hat{v}_{Si}\|$.
The proof is complete.
\end{proof}

\begin{lemma}
Given a linearly separable dataset $S=\{(x_1,y_1),(x_2,y_2),\cdots,(x_n,y_n)\}$, $x_i\in \mathbb{R}^{d_x}, \|x_i\|\le B, y_i\in \{-1,1\}$. We denote $\hat{w}_{S}=[\hv_{S}^\T,\hb_{S}]^\T$ as the SVM parameters trained on dataset $S$ and denote $\hat{w}_{S^i}=[\hv_{S^i}^\T,\hb_{S^i}]^\T$ as the SVM parameters trained on dataset $S^i$. All input vectors $x_i$ is bounded as $\|x_i\|\le B$. We denote $\gamma_S$ as the maximal margin between the separation plan $\hv_S^{\T}x-\hb_S=0$ and encoded features $x_i$. For some constant $\lambda$, we have
$$\|\hat{w}_{S^i}-\hat{w}_S\|\le \max\big\{\sqrt{\frac{2}{\lambda n}\left(1+\frac {B} {\gamma_S}\right)}, \frac{B+\sqrt{B^2+8n\lambda(1+B/\gamma_S)}}{2n\lambda} \big\}.$$
\label{lemma:svmdis}
\end{lemma}

\begin{proof}
Since $\hat{w}_{S}=[\hv_S,\hb_S]$ minimizes the SVM optimization target function $$\begin{aligned}
    \LL_S(v,b)=\lambda \lVert v \rVert^2 +\left[\frac 1 n \sum_{i=1}^n \max\left(0, 1 - y_i(v^\mathsf{T} x_i - b)\right) \right]
\end{aligned}$$
while $\hat{w}_{S^i}=[\hv_{S^i},\hb_{S^i}]$ minimizes the SVM optimization target function $$\begin{aligned}
    \LL_{S^i}(v,b)=\lambda \lVert v \rVert^2 +\left[\frac 1 {n-1} \sum_{j=1,j\ne i}^n \max\left(0, 1 - y_j(v^\mathsf{T} x_j - b)\right) \right] 
\end{aligned}$$

We denote the function $\tilde{L}_S(v,b)=\frac 1 n \sum_{i=1}^n \max\left(0, 1 - y_i(v^\mathsf{T} x_i - b)\right)$ on dataset $S$. From \Cref{lemma:cvx}, $\tilde{L}_S$ is a convex function while $(v,b) \rightarrow\lambda\|v\|^2$ is a 2$\lambda$-strongly convex. Therefore, $\LL_S(v,b)$ is also a 2$\lambda$-strongly convex function. Then, with the property of the strongly convex function, $\forall w=[v^{\T},b]^{\T}$, we have
\begin{equation}
    \LL_S(w)-\LL_S(\hat{w}_S)\ge \lambda\|w-\hat{w}_S\|^2.
    \label{eqn:scp}
\end{equation}
On the other hand, 
$$\begin{aligned}
    \LL_S(\hat{w}_{S^i})-\LL_S(\hat{w}_S)&=\tilde{L}_S(\hat{w}_{S^i})+\lambda\|\hat{w}_{S^i}\|^2-(\tilde{L}_S(\hat{w}_S)+\lambda\|\hat{w}_S\|^2)\\
    &\approx \tilde{L}_{S^i}(\hat{w}_{S^i})+\lambda\|\hat{w}_{S^i}\|^2-(\tilde{L}_{S^i}(\hat{w}_S)+\lambda\|\hat{w}_S\|^2)+\frac{\ell(\hat{w}_{S^i},x_i)-\ell(\hat{w}_{S},x_i)}{n} && (\text{as }n\approx n-1)\\
    &\le \frac{\ell(\hat{w}_{S^i},x_i)-\ell(\hat{w}_{S},x_i)}{n}\\
    &= \frac{\max(0, 1 - y_i(\hv^\T_{S^i} x_i - \hb_{S^i}))-\max(0, 1 - y_i(\hv^\T_{S} x_i - \hb_{S}))}{n}\\
    &\le \frac{1}{n}\max(1+\|\hv^\T_{S^i}\|\|x_i\|+|\hb_{S^i}|,1+\|\hv^\T_{S}\|\|x_i\|+|\hb_{S}|,\|\hv^\T_{S^i}-\hv^\T_{S}\|\|x_i\|+|\hb_{S^i}|+|\hb_{S}|)&& \text{(\Cref{lemma:cvx})} \\
    &\le \frac{1}{n} \biggl(\max(1+\|\hv^\T_{S^i}\|\|x_i\|+1+B\|\hv^\T_{S^i}\|,1+\|\hv^\T_{S}\|\|x_i\|+1+B\|\hv^\T_{S}\|,\\
    &\ \ \ \ \ \ \ \ \ \ \ \|\hv^\T_{S^i}-\hv^\T_{S}\|\|x_i\|+1+B\|\hv^\T_{S}\|+1+B\|\hv^\T_{S^i}\|\biggr)&& \text{(\Cref{lemma:bbound})} \\ 
    &\le \frac{1}{n} \biggl(\max(2+2\|\hv^\T_{S^i}\|B,2+2\|\hv^\T_{S}\|B, 2+\|\hv^\T_{S^i}-\hv^\T_{S}\|B+B\|\hv^\T_{S}\|+B\|\hv^\T_{S^i}\|\biggr) \\
    &\le \frac{1}{n} \biggl(\max(2+2\|\hv^\T_{S}\|B, 2+\|\hv^\T_{S^i}-\hv^\T_{S}\|B+2B\|\hv^\T_{S}\|\biggr) && \text{(\Cref{lemma:maxw})} \\
\end{aligned}$$

Plug this bound into \Cref{eqn:scp}, we have:

If $2+2\|\hv^\T_{S}\|B \ge 2+\|\hv^\T_{S^i}-\hv^\T_{S}\|B+2B\|\hv^\T_{S}\|$:

$$\begin{aligned}
\lambda\|\hat{w}_{S^i}-\hat{w}_S\|^2 &\le \frac{1}{n} \biggl(\max(2+2\|\hv^\T_{S}\|B, 2+\|\hv^\T_{S^i}-\hv^\T_{S}\|B+2B\|\hv^\T_{S}\|\biggr)=\frac{2+2\|\hv^\T_{S}\|B}{n}\\
\|\hat{w}_{S^i}-\hat{w}_S\| &\le \sqrt{\frac{(2+2B\|\hat{v}_{S}\|)}{\lambda n}}=\sqrt{\left(2+\frac {2B} {\gamma_S}\right)\frac{1}{\lambda n}}.
\end{aligned}
$$

If $2+2\|\hv^\T_{S}\|B < 2+\|\hv^\T_{S^i}-\hv^\T_{S}\|B+2B\|\hv^\T_{S}\|$:

$$\begin{aligned}
    \lambda\|\hat{w}_{S^i}-\hat{w}_S\|^2 &\le \frac{1}{n} \biggl(\max(2+2\|\hv^\T_{S}\|B, 2+\|\hv^\T_{S^i}-\hv^\T_{S}\|B+2B\|\hv^\T_{S}\|\biggr)=\frac{2+\|\hv^\T_{S^i}-\hv^\T_{S}\|B+2B\|\hv^\T_{S}\|}{n}\\
    \Rightarrow\ \ \ \ \ \  &n\lambda\|\hat{w}_{S^i}-\hat{w}_S\|^2-B\|\hat{w}_{S^i}-\hat{w}_S\|-(2+2B\|\hv^\T_{S}\|)\le 0.
    \end{aligned}
    $$
    This is a quadratic function of $\|\hat{w}_{S^i}-\hat{w}_S\|$, the maximal feasible value for $\|\hat{w}_{S^i}-\hat{w}_S\|$ can be calculated as
    $$\|\hat{w}_{S^i}-\hat{w}_S\|\le \frac{B+\sqrt{B^2+8n\lambda(1+B/\gamma_S)}}{2n\lambda}$$

    Therefore, 

    $$\|\hat{w}_{S^i}-\hat{w}_S\|\le \max\big\{\sqrt{\frac{2}{\lambda n}\left(1+\frac {B} {\gamma_S}\right)}, \frac{B+\sqrt{B^2+8n\lambda(1+B/\gamma_S)}}{2n\lambda} \big\}$$

It should be noted that, here, $x_i$ is equivalent to $E(x_i)$ in the main paper. We just use $x_i$ to simplify the notation. 
The proof is complete.

\end{proof}

\mainthm*

\begin{proof}

By utilizing the telescope sum, we have
$$\begin{aligned}
\left|\frac{w_{S^i}}{\|w_{S^i}\|}-\frac{w_{S}}{\|w_{S}\|}\right| &= \left|\frac{w_{S^i}}{\|w_{S^i}\|}-\frac{\hat{w}_{S^i}}{\|\hat{w}_{S^i}\|}-(\frac{w_{S}}{\|w_{S}\|}-\frac{\hat{w}_{S}}{\|\hat{w}_{S}\|})+\frac{\hat{w}_{S^i}}{\|\hat{w}_{S^i}\|}-\frac{\hat{w}_{S}}{\|\hat{w}_{S}\|}\right| \\
& \le \left|\frac{w_{S^i}}{\|w_{S^i}\|}-\frac{\hat{w}_{S^i}}{\|\hat{w}_{S^i}\|}\right|+\left|\frac{w_{S}}{\|w_{S}\|}-\frac{\hat{w}_{S}}{\|\hat{w}_{S}\|}\right|+\left|\frac{\hat{w}_{S^i}}{\|\hat{w}_{S^i}\|}-\frac{\hat{w}_{S}}{\|\hat{w}_{S}\|}\right|\\
& \le \frac{C\log \log t}{\log t}+\left|\frac{\hat{w}_{S^i}}{\|\hat{w}_{S^i}\|}-\frac{\hat{w}_{S}}{\|\hat{w}_{S}\|}\right| && \text{(Theorem 5 \cite{soudry2018implicit})}\\
&\le \frac{C\log \log t}{\log t}+\nu\|\hat{w}_{S^i}-\hat{w}_{S}\| && \text{(\Cref{lemma:mlip})}\\
&\le \frac{C\log \log t}{\log t}+\nu \max\big\{\sqrt{\frac{2}{\lambda n}\left(1+\frac {B} {\gamma_S}\right)}, \frac{B+\sqrt{B^2+8n\lambda(1+B/\gamma_S)}}{2n\lambda} \big\} && \text{(\Cref{lemma:svmdis})}
\end{aligned}$$
The proof is complete.
\end{proof}

\section{Proof of \Cref{thm:thmlnsr}}\label{sec:A-thmlnsr}

\thmlnsr*

\begin{proof}
With a similar idea as \citet{shalev2014understanding}, we have
$$\begin{aligned}
    \LL_S(\hat{w}_{S^i})-\LL_S(\hat{w}_S)&=\tilde{L}_S(\hat{w}_{S^i})+\lambda\|\hat{w}_{S^i}\|^2-(\tilde{L}_S(\hat{w}_S)+\lambda\|\hat{w}_S\|^2)\\
    &\approx \tilde{L}_{S^i}(\hat{w}_{S^i})+\lambda\|\hat{w}_{S^i}\|^2-(\tilde{L}_{S^i}(\hat{w}_S)+\lambda\|\hat{w}_S\|^2)+\frac{\ell(\hat{w}_{S^i},x_i)-\ell(\hat{w}_{S},x_i)}{n}\\
    &\le \frac{\ell(\hat{w}_{S^i},x_i)-\ell(\hat{w}_{S},x_i)}{n}\\
    &\le \frac{L \|\hat{w}_{S}-\hat{w}_{S^i}\|}{n}\\
\end{aligned}$$
Plug this bound into \Cref{eqn:scp}, we have:
$$\begin{aligned}
    \lambda\|\hat{w}_{S^i}-\hat{w}_S\|^2 &\le \frac{L \|\hat{w}_{S}-\hat{w}_{S^i}\|}{n}\\
    \|\hat{w}_{S^i}-\hat{w}_S\| &\le \frac{L }{\lambda n}
    \end{aligned}\\$$

Similar to \Cref{thm:mainthm}, we have
$$\E_{S}\left[\left\|\frac{\A(S^{i})}{\|\A(S^{i})\|}-\frac{\A(S)}{\|\A(S)\|}\right\|\right]\le \frac{C\log \log t}{\log t}+\nu \frac{L }{\lambda n}$$
The proof is complete.
\end{proof}

\section{Proof of \Cref{thm:thmmhl}}\label{sec:A-thmmhl}

To prove \Cref{thm:thmmhl}, we first prove \Cref{lemma:exp0} which shows that the expectation of the direction of the weight directs at the direction of the SVM separation plane.

\begin{lemma}
If the dataset $S$ is linearly separable, and the linear model is optimized by gradient descent method with $t$ iteration steps and the stepsize $\eta< 2\beta^{-1}\sigma_{\max}^{-1}(X )$. 
$$\E \frac{ w(t)}{\| w(t)\|}=\frac{\hat{w}_{S}}{\|\hat{w}_{S}\|} $$
\label{lemma:exp0}
\end{lemma}
\begin{proof}
We denote $\E \nabla \LL(w(t))=g_w$.
As indicated by \citet{soudry2018implicit}, $\|w (t) \|\rightarrow \infty$ and thus $w(0)$ can be ignored. By the definition of the gradient descent, we have
$$\begin{aligned}
    w(t+1)&=w(t)-\eta \nabla \LL (w(t))\\
    w(\infty)&=w(0)-\eta \nabla \LL (w(1))-\eta \nabla \LL (w(2))-\cdots\\
    &\approx-\eta \nabla \LL (w(1))-\eta \nabla \LL (w(2))-\cdots\\
    &=-\eta \sum_{i=1}^\infty\nabla \LL (w(i))
\end{aligned}$$
Following \citet{soudry2018implicit}, we have
$$\begin{aligned}  
    \frac{\hat{w}_S}{\|\hat{w}_S\|}&=\frac{w(\infty)}{\|w(\infty)\|}\\
    &=\frac{-\eta \sum_{i=1}^\infty\nabla \LL (w(i))}{\|-\eta \sum_{i=1}^\infty\nabla \LL (w(i))\|}\\
    &=\frac{-\eta \frac 1 N  \sum_{i=1}^\infty\nabla \LL (w(i))}{\|-\eta \frac 1 N  \sum_{i=1}^\infty\nabla \LL (w(i))\|}\\
    &=\frac{-\eta \frac 1 N \lim_{N\rightarrow \infty} \sum_{i=1}^N\nabla \LL (w(i))}{\|-\eta \frac 1 N \lim_{N\rightarrow \infty} \sum_{i=1}^N\nabla \LL (w(i))\|}\\
    &=\frac{- \E \nabla \LL(w(t))}{\| \E \nabla \LL(w(t))\|}\\
    &=\frac{- g_w}{\| g_w\|}\\
\end{aligned}$$
$$\begin{aligned}
    \E w(t)&=\E [w(0)-\eta \nabla \LL (w(1))-\eta \nabla \LL (w(2))-\cdots-\eta \nabla \LL (w(t-1))]\\
    &\approx\E [-\eta \nabla \LL (w(1))-\eta \nabla \LL (w(2))-\cdots-\eta \nabla \LL (w(t-1))]\\
    &= -\eta \E\nabla \LL (w(1))-\eta \E\nabla  \LL (w(2))-\cdots-\eta \E\nabla \LL (w(t-1))\\
    &= -\eta g_w-\eta g_w-\cdots-\eta g_w\\
    &= -\eta(t-1) g_w\\
    &= -\eta(t-1) \|g_w\|(-\frac{\hat{w}_S}{\|\hat{w}_S\|})\\
\end{aligned}$$
Therefore, we have the expectation $\E w(t)=D \hat{w}_S$, where $D=\eta(t-1) \|g_w\|\frac{1}{\|\hat{w}_S\|}$ is a non-negative scalar. Therefore,
$$\E \frac{ w(t)}{\| w(t)\|}=\frac{\hat{w}_{S}}{\|\hat{w}_{S}\|} $$
The proof is complete.
\end{proof}

\thmmhl*

\begin{proof}
We denote $w_h(t)$ as the weight for head $h$ at time $t$. We denote the projection of $w_h(t)$ onto $w_S$ as $$\tilde{w}_h(t)=\frac{w_Sw_S^{\T}}{w_S^{\T}w_S}w_h(t).$$ We also denote the projection on the orthogonal space of $w_S$ as $$r_h(t)=(I-\frac{w_Sw_S^{\T}}{w_S^{\T}w_S})w_h(t),$$ 
where $I$ is the identity matrix. It can be easily verfied that $w_h(t)=\tilde{w}_h(t)+r_h(t)$ and $\tilde{w}_h(t)^{\T}r_h(t)=0$. As the expectation of $w_h(t)$ is along the direction of $w_S$ while $\|w_h(t)\|\rightarrow\infty$, the  residual norm $\|r_h(t)\|$ is negligible compared with the projection norm $\|\tilde{w}_h(t)\|$, namely, $\|r_h(t)\|\ll\|\tilde{w}_h(t)\|$. Then, we have

$$\begin{aligned}
    \left\|\frac{w_h(t)}{\|w_h(t)\|}-\frac{w_S}{\|w_S\|}\right\|&=\left\|\frac{\tilde{w}_h(t)+r_h(t)}{\|\tilde{w}_h(t)+r_h(t)\|}-\frac{w_S}{\|w_S\|}\right\|\\
    &=\left\|\frac{\tilde{w}_h(t)+r_h(t)}{\|\sqrt{\|\tilde{w}_h(t)\|^2+2\tilde{w}_h(t)^{\T}r_h(t)+\|r_h(t)\|^2}}-\frac{w_S}{\|w_S\|}\right\|\\
    &=\left\|\frac{\tilde{w}_h(t)+r_h(t)}{\|\sqrt{\|\tilde{w}_h(t)\|^2+\|r_h(t)\|^2}}-\frac{w_S}{\|w_S\|}\right\|\\
    &\approx\left\|\frac{\tilde{w}_h(t)+r_h(t)}{\|\sqrt{\|\tilde{w}_h(t)\|^2}}-\frac{w_S}{\|w_S\|}\right\| && (\text{as }\|r_h(t)\|\ll\|\tilde{w}_h(t)\|)\\
    &=\left\|\frac{\tilde{w}_h(t)}{\|\tilde{w}_h(t)\|}+\frac{r_h(t)}{\|\tilde{w}_h(t)\|}-\frac{w_S}{\|w_S\|}\right\|\\
    &=\left\|\frac{r_h(t)}{\|\tilde{w}_h(t)\|}\right\|\\
\end{aligned}$$
Therefore, from the results of \citet{soudry2018implicit}, we have
\begin{equation}
    \left\|\frac{r_h(t)}{\|\tilde{w}_h(t)\|}\right\|\approx \left\|\frac{w_h(t)}{\|w_h(t)\|}-\frac{w_S}{\|w_S\|}\right\|\le \frac{C\log \log t}{\log t}
    \label{eqn:diffbound}
\end{equation}

On the other hand, for the multi-head classifier weight $\bar{w}_h(t)$, we have $$\bar{w}_h(t)=\frac 1 H \sum_{h=1}^Hw_h(t)=\frac 1 H \sum_{h=1}^H \tilde{w}_h(t)+r_h(t)$$

We use $\xi_h\ge 1$ as a adjusting factor such that $$\xi_h=\frac{H\|\tilde{w}_h(t)\|}{\|\sum_{h=1}^H \tilde{w}_h(t)\|}$$

As assumed in this Theorem that $\bar{w}^\T\hw_S\ne 0$, it implies that $\sum_{h=1}^H \tilde{w}_h(t)\ne0$. Therefore, all $\xi_h$s are uniformly bounded as 
$$\xi_h\le \xi=\frac{H\max_h\|\tilde{w}_h(t)\|}{\|\sum_{h=1}^H \tilde{w}_h(t)\|},$$

where $\xi$ is a finite constant.

Then, we can get
$$\begin{aligned}
    \left\|\frac{\bar{w}_h(t)}{\|\bar{w}_h(t)\|}-\frac{w_S}{\|w_S\|}\right\|&=\left\|\frac{\frac 1 H \sum_{h=1}^H \tilde{w}_h(t)+r_h(t)}{\|\frac 1 H \sum_{h=1}^H \tilde{w}_h(t)+r_h(t)\|}-\frac{w_S}{\|w_S\|}\right\|\\
    &=\left\|\frac{ \sum_{h=1}^H \tilde{w}_h(t)+r_h(t)}{\|\sum_{h=1}^H \tilde{w}_h(t)+r_h(t)\|}-\frac{w_S}{\|w_S\|}\right\|\\
    &=\left\|\frac{ \sum_{h=1}^H \tilde{w}_h(t)+r_h(t)}{\sqrt{\|\sum_{h=1}^H \tilde{w}_h(t)\|^2+2(\sum_{h=1}^H \tilde{w}_h(t))^{\T}(\sum_{h=1}^Hr_h(t))+\|\sum_{h=1}^Hr_h(t)\|^2}}-\frac{w_S}{\|w_S\|}\right\|\\
    &=\left\|\frac{ \sum_{h=1}^H \tilde{w}_h(t)+r_h(t)}{\sqrt{\|\sum_{h=1}^H \tilde{w}_h(t)\|^2+\|\sum_{h=1}^Hr_h(t)\|^2}}-\frac{w_S}{\|w_S\|}\right\|\\
    &\approx\left\|\frac{ \sum_{h=1}^H \tilde{w}_h(t)+r_h(t)}{\|\sum_{h=1}^H \tilde{w}_h(t)\|}-\frac{w_S}{\|w_S\|}\right\|\\
    &=\left\|\frac{ \sum_{h=1}^H \tilde{w}_h(t)}{\|\sum_{h=1}^H \tilde{w}_h(t)\|}+\frac{ \sum_{h=1}^H r_h(t)}{\|\sum_{h=1}^H \tilde{w}_h(t)\|}-\frac{w_S}{\|w_S\|}\right\|\\
    &=\left\|\frac{ \sum_{h=1}^H r_h(t)}{\|\sum_{h=1}^H \tilde{w}_h(t)\|}\right\|\\
    &=\left\|\sum_{h=1}^H \frac{1}{H}\xi_h\frac{  r_h(t)}{\| \tilde{w}_h(t)\|}\right\|\\
    &=\left\|\frac{1}{H}\sum_{h=1}^H \xi_h\frac{  r_h(t)}{\| \tilde{w}_h(t)\|}\right\|
\end{aligned}$$

As we also have 
$$\begin{aligned}
    \E \xi_h\frac{  r_h(t)}{\| \tilde{w}_h(t)\|} &=0 && \text{(\Cref{lemma:exp0})}\\
    \E \|\xi_h\frac{  r_h(t)}{\| \tilde{w}_h(t)\|} \|^2&\le \left(\frac{C\xi\log \log t}{\log t}\right)^2, && \text{(\Cref{eqn:diffbound})}
\end{aligned}
$$

Then, By Lemma 18  \citet{kohler2017sub}, we can conclude that with probability $1-\delta$, (here $\delta > \exp(\frac 1 4 - \frac H 8)$) we have
$$\left\|\frac{\bar{w}_h(t)}{\|\bar{w}_h(t)\|}-\frac{w_S}{\|w_S\|}\right\|\approx \left\|\frac{1}{H}\sum_{h=1}^H \xi_h\frac{  r_h(t)}{\| \tilde{w}_h(t)\|}\right\| \le \sqrt{\frac{8(\frac 1 4 +\log \frac 1 \delta)}{H}}\frac{C\xi\log \log t}{\log t} $$

It should be noted that Lemma 18 \cite{kohler2017sub} is a Bernstein inequality and it holds when the probability $\delta$ is larger than a certain quantity which depends on the sample count $n$. For a fixed $\delta$, we require $H>2+8\ln\frac 1 \delta$. Therefore, with probability $1-\delta$, we have

$$\begin{aligned}
    \left|\frac{\bar{w}_{S^i}}{\|\bar{w}_{S^i}\|}-\frac{\bar{w}_{S}}{\|\bar{w}_{S}\|}\right| &= \left|\frac{\bar{w}_{S^i}}{\|\bar{w}_{S^i}\|}-\frac{\hat{w}_{S^i}}{\|\hat{w}_{S^i}\|}-(\frac{\bar{w}_{S}}{\|\bar{w}_{S}\|}-\frac{\hat{w}_{S}}{\|\hat{w}_{S}\|})+\frac{\hat{w}_{S^i}}{\|\hat{w}_{S^i}\|}-\frac{\hat{w}_{S}}{\|\hat{w}_{S}\|}\right| \\
    & \le \left|\frac{\bar{w}_{S^i}}{\|\bar{w}_{S^i}\|}-\frac{\hat{w}_{S^i}}{\|\hat{w}_{S^i}\|}\right|+\left|\frac{\bar{w}_{S}}{\|\bar{w}_{S}\|}-\frac{\hat{w}_{S}}{\|\hat{w}_{S}\|}\right|+\left|\frac{\hat{w}_{S^i}}{\|\hat{w}_{S^i}\|}-\frac{\hat{w}_{S}}{\|\hat{w}_{S}\|}\right| \\
    & \le \sqrt{\frac{8(\frac 1 4 +\log \frac 1 \delta)}{H}}\frac{C\xi\log \log t}{\log t}+\left|\frac{\hat{w}_{S^i}}{\|\hat{w}_{S^i}\|}-\frac{\hat{w}_{S}}{\|\hat{w}_{S}\|}\right| \\
    &\le \sqrt{\frac{8(\frac 1 4 +\log \frac 1 \delta)}{H}}\frac{C\xi\log \log t}{\log t}+\nu\|\hat{w}_{S^i}-\hat{w}_{S}\| && \text{(\Cref{lemma:mlip})}\\
    &\le \sqrt{\frac{8(\frac 1 4 +\log \frac 1 \delta)}{H}}\frac{C\xi\log \log t}{\log t} +\nu \max\big\{\sqrt{\frac{2}{\lambda n}\left(1+\frac {B} {\gamma_S}\right)}, \frac{B+\sqrt{B^2+8n\lambda(1+B/\gamma_S)}}{2n\lambda} \big\} && \text{(\Cref{lemma:svmdis})}
\end{aligned}$$

The proof is complete.
\end{proof}

\section{Miscellaneous}\label{sec:A-descentlemma}

\begin{lemma}[Descent Lemma]
Suppose $f$ is $\beta$-smooth. If the learning rate $\eta<1 / \beta$, we have
$$
f\left(w_{t+1}\right) \leq f\left(w_t\right)-\frac{\eta}{2} \cdot\left\|\nabla f\left(w_t\right)\right\|_2^2
$$
\label{lemma:descent}
\end{lemma}

\begin{proof}

$$
\begin{aligned}
f\left(w_{t+1}\right) &=f\left(w_t-\eta \nabla f\left(w_t\right)\right) \\
& \leq f\left(w_t\right)+\left\langle\nabla f\left(w_t\right),-\eta \nabla f\left(w_t\right)\right\rangle+\frac{1}{2}\eta \nabla f\left(w_t\right)^T \nabla^2 f\left(w_t\right)\eta \nabla f\left(w_t\right)  \\
& \leq f\left(w_t\right)+\left\langle\nabla f\left(w_t\right),-\eta \nabla f\left(w_t\right)\right\rangle+\frac{\beta}{2}\left\|\eta \nabla f\left(w_t\right)\right\|_2^2  \\
&=f\left(w_t\right)-\frac{\left(\eta-\eta^2 \beta\right)}{2}\left\|\nabla f\left(w_t\right)\right\|_2^2 \\
& \leq f\left(w_t\right)-\frac{\eta}{2} \cdot\left\|\nabla f\left(w_t\right)\right\|_2^2\ \ \ \ \ \ \ \ \ \ \ \ \  (\text{As }\eta \leq 1 / \beta)
\end{aligned}
$$
The proof is complete.

\end{proof}

\begin{figure}[h]
    \centering
    \includegraphics[width=1.\columnwidth]{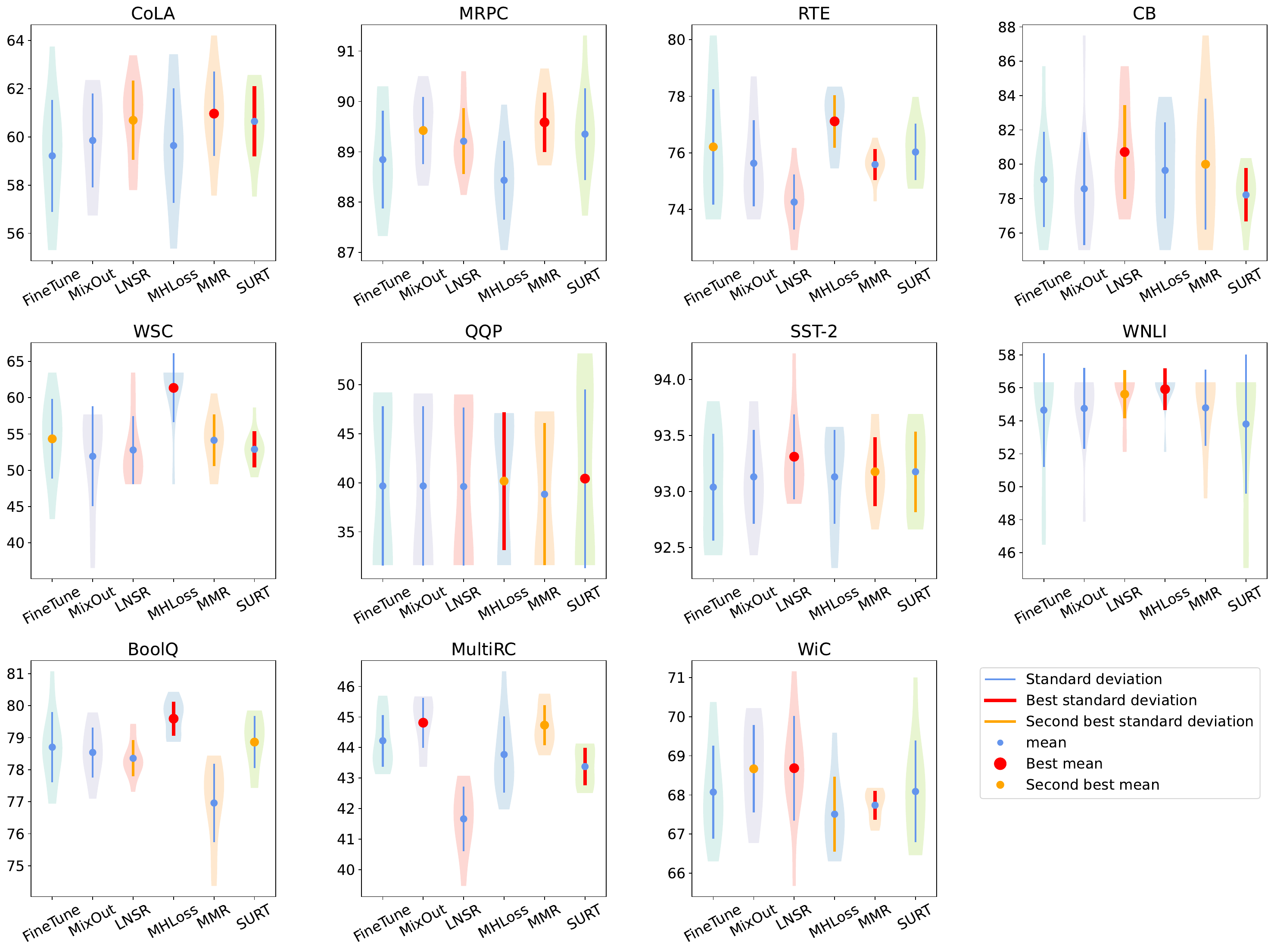}
    \caption{Violin plot for the main experiment.}
    \label{fig:varbox} 

\end{figure}

\section{Violin Plot for the Main Experiment}\label{sec:varbox}
To better illustrate the score distribution of the main experiment in \Cref{tab:main}. We draw the violin plot for each task as shown in \Cref{fig:varbox} which shows the probability density of the data at different values. From the figure, we can draw the same conclusions as that in the main experiment.

\section{Distance between $w_0$ and $w_*$}\label{sec:pardist}
In \Cref{sec:fulltune}, we employ Taylor expansion to approximate the original function $f(w,x)$ at the point $w_0$. To validate the reliability of Taylor expansion, we first demonstrate the proximity between the finetuned parameter $w_*$ and the pre-trained parameter $w_0$. We calculate the distance between $w_0$ and $w_*$ as $\|w_*-w_0\|$ and normalize it with the norm of $w_0$ to better reflect the parameter changes in $w_*$. We report the 1-norm, 2-norm, and $\infty$-norm for the weight vectors. The results are presented in \Cref{tab:taylor}. As shown in the results, after fine-tuning the model, the parameter the new parameter $w_*$ is very close to the pre-trained parameter $w_0$, changing $w_0$ by only 1\% in terms of the norm. This observation is also consistent with the findings reported by \citet{radiya2020fine}.

\begin{table}[t]
    \centering
    \scriptsize
    
    \begin{tabular}{lccc}
        \toprule
        {} & $\|w_*-w_0\|_2/\|w_0\|_2$ & $\|w_*-w_0\|_1/\|w_0\|_1$ & $\|w_*-w_0\|_\infty/\|w_0\|_\infty$ \\
        \midrule
        CoLA  &                    0.0101 &                    0.0095 &                              0.0205 \\
        MRPC  &                    0.0117 &                    0.0114 &                              0.0235 \\
        SST-2 &                     0.021 &                    0.0206 &                              0.0219 \\
        CB    &                    0.0074 &                    0.0063 &                               0.023 \\
        WNLI  &                     0.021 &                    0.0198 &                              0.0225 \\
        RTE   &                    0.0154 &                    0.0151 &                              0.0221 \\
        \bottomrule
        \end{tabular}
\caption{Distance between $w_0$ and $w_*$. We regularize the corresponding distance by dividing the norm of $w_0$.}
\label{tab:taylor}
\end{table}

\begin{figure}[t]
    \centering
    \includegraphics[width=0.99\columnwidth]{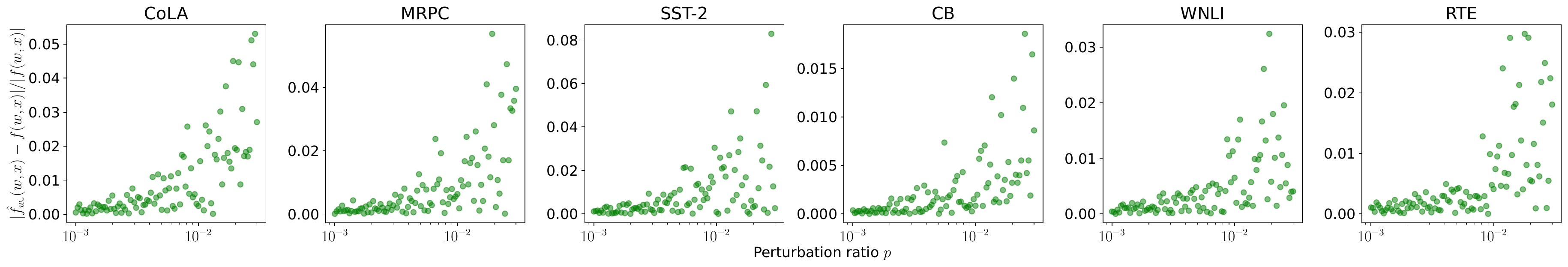}
    \caption{Validity of Taylor expansion.}
    \label{fig:taylor}
\end{figure}

\section{The Validity of Taylor Expansion}\label{sec:valtaylor}
Applying Taylor expansion approximation in theoretical analysis of neural networks has been used in many recent works \cite{liu2018towards,huang2020improving,foret2021sharpness,xu2021raise,barrett2021implicit,hoefler2021sparsity,smith2021origin,wen2022does,zuo2022moebert,shi2022gradient}. To further verify the validity of the Taylor expansion, we conducted a new experiment. For a fine-tuned model with parameter $w_*$, we first used a second-order Taylor expansion to expand the functions near $w_*$ as shown in Equation \ref{eqn:taylorexp}:

\begin{equation}
\hat{f}_{w_*}(w,x)=f(w_*,x)+(w-w_*)^{\T}\nabla f(w_*,x)+\frac{1}{2}(w-w_*)^{\T}\nabla^2f(w_*,x)(w-w_*),
\label{eqn:taylorexp}
\end{equation} 

where $\hat{f}_{w_*}(w,x)$ is short for the second-order Taylor expansion. Unfortunately, directly calculating the Hessian for a neural network is impossible as it requires storing $n^2$ parameters, where $n$ is the number of model parameters. For a simple Bert model with about 108 million parameters, the Hessian requires storing 11 quadrillion parameters, which is approximately 170,000 TB of memory and is infeasible to be stored in any device. On the other hand, though Hessian $\nabla^2f(w_*,x)$ is hard to calculate, the matrix-vector product $\nabla^2f(w_*,x)r$ is much cheaper to obtain \cite{bishop2006pattern} because we can take the derivative of $\nabla f(w_*,x)^\T r$ to directly get the results which can be denoted as
\begin{equation}
  \small
  \frac{\partial \nabla f(w_*,x)^\T r}{\partial w}=\frac{\partial \nabla f(w_*,x)^\T}{\partial w}r+\nabla f(w_*,x)^\T\frac{\partial r}{\partial w}=\frac{\partial \nabla f(w_*,x)^\T}{\partial w}r+0=\nabla^2f(w_*,x)r,
\end{equation}
where $r$ can be an arbitrarily given vector. Such an approach has already been implemented by most of the automatic differentiation softwares including Pytorch \cite{paszke2019pytorch}, Tensorflow \cite{abadi2016tensorflow}, JAX \cite{bradbury2018jax}, etc. Based on this observation, we first calculate the Hessian vector product $\nabla^2f(w_*,x)(w-w_*)$. Then, we can calculate each term in Equation \ref{eqn:taylorexp} to obtain the Taylor approximation. We conduct our experiment on 6 GLUE tasks including CoLA, MRPC, SST-2, CB, WNLI, and RTE. In each task, we first randomly perturb the model parameters $w_*$ as $w=w_*(1 + p(\mathcal{U} - 0.5))$ for each parameter,  where $p$ is a perturbation ratio parameter indicating how far $w$ is from $w_*$ and $\mathcal{U}$ is a uniformly distributed variable. Then we use the perturbed parameter $w$ to calculate the loss function denoted as $f(w,x)$. Besides, we also calculate the second Taylor expansion at $w_*$ for $w$ as $\hat{f}_{w_*}(w,x)$. Finally, we calculate the difference between the Taylor approximation and the real function value on $w$ as $|\hat{f}_{w_*}(w,x)-f(w,x)|$ and normalize it by dividing $|f(w,x)|$ to better reflect how the loss changes after using the Taylor expansion. We repeat this procedure on each task 100 times and the results are shown in \Cref{fig:taylor} where the x-axis is the perturbation ratio while the y-axis is the normalized loss distance between the real loss on $w$ and its Taylor approximation.
It can be observed from the results that the normalized loss gap $|\hat{f}_{w_*}(w,x)-f(w,x)|/|f(w,x)|$ increases as the perturbation ratio $p$ increases. This is because if $p$ is large, the new model with $w$ will be further from the fine-tuned model $w_*$. On the other hand, if the perturbation is controlled within 3\% of the parameter $w_*$, the loss for the new model $w$ will only change by no more than 5\% of the fine-tuned model.
Combined with the observation in Section \ref{sec:pardist}, that the model parameters only change around 1\%, we can come to an empirical conclusion that using Taylor expansion gives an acceptable approximation of the function around $w_*$. This conclusion verifies the validity of our assumption in the theoretical analysis.

\end{document}